\documentclass[lettersize,journal]{IEEEtran}
\usepackage{amsmath,amsfonts}
\newtheorem{theorem}{Theorem}
\newtheorem{remark}{Remark}

\newtheorem{proof}{Proof}[section]
\usepackage{algorithm}
\usepackage{algorithmic}
\usepackage{array}
\usepackage[caption=false,font=normalsize,labelfont=sf,textfont=sf]{subfig}
\usepackage{textcomp}
\usepackage{cite} 
\usepackage{stfloats}
\usepackage{url}
\usepackage{verbatim}
\usepackage{graphicx}
\usepackage{multirow}
\usepackage[normalem]{ulem}
\useunder{\uline}{\ul}{}
\usepackage{mathrsfs}

\def\BibTeX{{\rm B\kern-.05em{\sc i\kern-.025em b}\kern-.08em
    T\kern-.1667em\lower.7ex\hbox{E}\kern-.125emX}}
\usepackage{balance}
\begin{document}
\title{AoI-based Temporal Attention Graph Neural Network for Popularity Prediction and Content Caching}
\author{\IEEEauthorblockN{Jianhang Zhu, Rongpeng Li, Guoru Ding, Chan Wang, Jianjun Wu, Zhifeng Zhao, and Honggang Zhang}

\thanks{
   J. Zhu, R. Li, C. Wang and H. Zhang are with the College of Information Science and Electronic Engineering, Zhejiang University, Hangzhou 310027, China (e-mail: \{zhujh20, lirongpeng, 0617464, honggangzhang\}@zju.edu.cn).
   
   G. Ding is College of Communications and Engineering, Army Engineering University of PLA, Nanjing China (e-mail: dr.guoru.ding@ieee.org).
   
   J. Jun is with Huawei Technologies Company, Ltd., Shanghai 201206, China. (e-mail: wujianjun@huawei.com).
   
   Z. Zhao is with Zhejiang Lab, Hangzhou, China as well as the College of Information Science and Electronic Engineering, Zhejiang University, Hangzhou 310027, China (e-mail: zhaozf@zhejianglab.com).
   
  }

}

\maketitle

\begin{abstract}
Along with the fast development of network technology and the rapid growth of network equipment, the data throughput is sharply increasing. To handle the problem of backhaul bottleneck in cellular network and satisfy people's requirements about latency, the network architecture like information-centric network (ICN) intends to proactively keep limited popular content at the edge of network based on predicted results. Meanwhile, the interactions between the content (e.g., deep neural network models, Wikipedia-alike knowledge base) and users could be regarded as a dynamic bipartite graph. In this paper, to maximize the cache hit rate, we leverage an effective dynamic graph neural network (DGNN) to jointly learn the structural and temporal patterns embedded in the bipartite graph. Furthermore, in order to have deeper insights into the dynamics within the evolving graph, we propose an age of information (AoI) based attention mechanism to extract valuable historical information while avoiding the problem of message staleness. Combining this aforementioned prediction model, we also develop a cache selection algorithm to make caching decisions in accordance with the prediction results. Extensive results demonstrate that our model can obtain a higher prediction accuracy than other state-of-the-art schemes in two real-world datasets. The results of hit rate further verify the superiority of the caching policy based on our proposed model over other traditional ways.
\end{abstract}


\section{Introduction}
\IEEEPARstart{G}{iven} the galloping number of users and mobile equipment \cite{cisco2020cisco}, the amount of data sharply surges and the wireless access points at the network edge confront the frequent congestion. Generally, besides video streaming, provisioning artificial intelligence (AI) and other dedicated network functions services are becoming the dominant factors that steer this explosion. Therefore, how to bring a better quality of user experiences (QoE) and quality of service (QoS) to users in the sharply growing data traffic under a constrained backhaul link is an intractable problem we have to face. Some resource-devouring approaches, such as higher frequency reuse, a larger scale of antennas or setting more bandwidth, can tackle this problem by increasing the capacity of cellular networks, but most of them fail to offer a durable solution in terms of scalability, costs and flexibility \cite{li2017cooperative}. On the other hand, some studies \cite{gu2020distributed, cha2007tube} point out that a tremendous data load comes from the repeated requests for a few same popular targets, especially the multimedia services at the edge, so does the scenario of AI service, in which the same AI samples or trained models usually may devote to numerous applications \cite{liao2020cognitive}. Therefore, aiming at dealing with these repetitive actions by storing part of fashionable content (e.g., deep neural network models, videos) at the network edge, it can obviously alleviate the backhaul traffic burden caused by the data explosion and greatly reduce the transmission delay or other issues \cite{yang2018content}. So there is a growing consensus that edge caching will play a prominent role in future communication systems and networks \cite{paschos2018role}.

\begin{figure}[t] 
\centering  
\includegraphics[scale = 0.48]{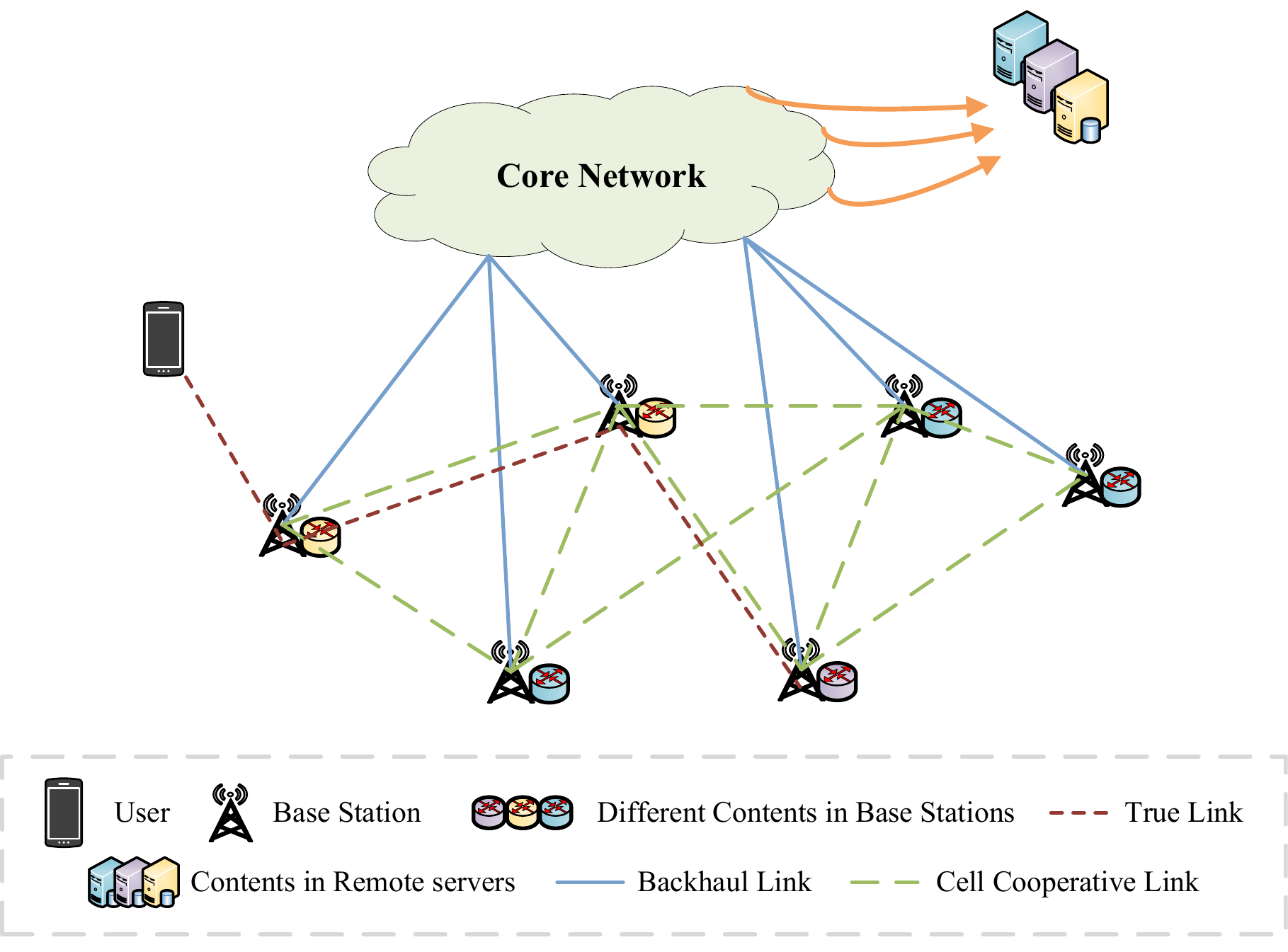}
\caption{Mobile edge caching in ICN.}  
\label{icn_fig}
\end{figure}

The demands for a more efficient and much simpler content distribution method have motivated the emergence of a new architecture called information-centric network (ICN) \cite{zhang2019sdn}. In contrast to the inefficient way like IP addressing, clients in ICN are able to directly access content pieces from the network only by its unique named data objects (NDO) \cite{lederer2014adaptive}. As illustrated in Fig. \ref{icn_fig}, ICN can easily satisfy requests by any edge node if the node holds a copy with the exact NDO in its in-network storage \cite{ahlgren2012survey}. Practically, due to the limited cache space \cite{somuyiwa2018reinforcement}, we are only able to recommend those content with distinguished cost performance to ICN's in-network storage. Ideally, we should proactively cache the most popular content. Most of the existing caching strategies always assume that the content popularity remains stable during a long period, while it actually varies over time \cite{chen2020content}. For example, some traditional caching strategies like Least Recently Used (LRU) and Least Frequently Used (LFU) \cite{lee2001lrfu} are partial to extract the superficial periodic law from historical information and ignore the dynamic characteristics of requests themselves \cite{dan1990approximate}. To enhance the caching performance, it becomes an incentive to lucubrate the dynamics to establish a popularity prediction model.

In addition to the dynamics of historical requests, we also believe a well-adopted popularity prediction model ought to predict from users' perspective and excavate the structural pattern within data as well. In other words, it is inspiring to exploit the implications from users with similar preferences when speculating some inactive entities' predilections for comprehensive popularity anticipation. Recently, some researchers regard user-content pairs in recommendation system (RS) as a bipartite graph \cite{zhou2020graph} and propose to utilize graph neural network (GNN), such as Graph Attention Network (GAT) \cite{velivckovic2017graph}, to dig out structured data. They demonstrate that even without dynamical information, GNN model also wins excellent performance in RS \cite{fan2019graph}. Inspired by these works \cite{zhou2020graph, velivckovic2017graph, fan2019graph}, interactions between the content and users could be also regarded as a dynamic bipartite graph when we attempt to predict their popularity in ICN. However, it is also non-negligible that most real-life graphs are always evolving, and the ignorance of time-varying nature in these aforementioned approaches make them still far from perfection \cite{Xu2020Inductive}. 

In order to realize the learning of structure and dynamics simultaneously, the dynamic graph neural network (DGNN) has been proposed. But existing approaches remain room to improve. In this article, we primarily leverage a modified continuous DGNN (CDGNN) model to learn the structural and temporal pattern in the dynamic bipartite graph of users and requested content. Specifically, we focus on discovering how to abstract temporal features and how many historical messages we should utilize while mining the dynamic features. To solve the above issues, we introduce age of information (AoI), a metric quantifying the freshness of data \cite{sun2017update}, to guide the selection of fresh information and a multi-head attention mechanism is employed to refine temporal characteristics. Both of them contribute to the generation of more precise representations for users' preference predictions. Afterwards, a strategy, which relies on the results of CDGNN is proposed to guide the caching. Overall, the main contributions of this paper are as follows:

\begin{itemize}
\item To forecast users’ preferences precisely, a CDGNN model is used to simultaneously excavate the structural and dynamical patterns of the bipartite graph for all users and design a caching policy on top of the CDGNN model.
\item We develop an AoI-based temporal attention graph neural network (ATAGNN) method by innovatively introducing the AoI concept and incorporating the attention mechanism on top of the GNN model to effectively mine the temporal features in the dynamic graph. The combination of AoI and GNN may be of independent interest to the GNN community.
\item Extensive simulation results also manifest the prediction accuracy of our method and confirm the superiority over other deep learning models. Meanwhile, we testify the strategy relying on our ATAGNN model's effectiveness within different cache spaces, updating periods. The results of caching hit rate show our scheme significantly surpasses the performance of traditional policies like LRU and LFU.
\end{itemize}

The reminder of this paper is organized as follows: The related work and background are introduced in Section \ref{sec2}. We present the system model in Section \ref{sec3}. The details of our proposed AoI-based temporal graph neural network with time generalization are delivered in Section \ref{sec4}. In Section \ref{sec5}, we provide the numerical analysis and the results of prediction. Finally, the conclusion is summarized in Section \ref{sec6}.

\section{Related Work}\label{sec2}

Traditional caching policies, LRU and LFU, have obtained motivating results. Furthermore, one of their variants \cite{ming2014age} proposes to combine age of information (AoI) with LRU to achieve cooperative caching between several cache-enabled edge servers, but due to the common limitation of LRU and LFU, it still fails to mine the dynamics under the superficial statistical law. To make a dynamic caching decision, utilizing the content popularity as a reference is a commonly-adopted way. Some innovative studies based on locally deployed popularity prediction algorithms thrive in recent years \cite{wu2014scaling, zhang2017ppc, mehrizi2019feature, lee2020t}. In particular, with the development of AI algorithms, deep learning plays a non-negligible role in the popularity prediction and caching task. For instance, \cite{chen2020content} uses a feed-forward neural network (FNN) for estimating the caching threshold to assist the caching decision. \cite{DBLP:journals/corr/HidasiKBT15} proposes to use a recurrent neural network (RNN) to recommend popular content, but it is complicated to be generalized to a new dataset that has not been trained because of the limitations of RNN. What's worse, both of them fail to exploit the interdependency among all users or the structural pattern of data. In recent years, due to the excellent performance of extracting the structural pattern of graph, some studies succeed in using graph neural network (GNN) to realize the popularity recommendation by regarding users and their requested items as the components of a bipartite graph \cite{wu2019session}. In that regard, GAT is verified to effectively learn the structure-rich representations by aggregating interrelated vertexes in the recommendation graph with weights calculated by attention mechanism \cite{wang2020multi}. 

Although GNN, GAT and their variants have yielded excellent performance in representation learning and won remarkable achievements in recommendation tasks, they neglect the impact of dynamic features and still have room to improve. Therefore, GNN models with dynamics learning have been proposed to bridge the gap. For example, DySat \cite{sankar2019dynamic} and most DGNN models at the early stage achieve their dynamics extraction by sampling series snapshots from the evolving graph with equal time intervals. But the choice of sample granularity is a prominent part of model designing, and an inappropriate granularity may result in the failure of yielding a snapshot with a new effective graph structure \cite{skardinga2021foundations}. To avoid this obstacle, some continuous dynamic graph models (CDGNN), e.g., DyRep \cite{trivedi2019dyrep}, propose to complete the graph computation with event sampling. Specifically, DyRep expresses dynamic graph as the evolution of structural and node communication with a recurrent architecture. Moreover, some researchers inspired by the position encoding in Transformer \cite{vaswani2017attention} try to inject some information of interaction timestamps into the node. For instance, TGAT \cite{Xu2020Inductive} uses the harmonic processing method to obtain a time coding function. Furthermore, TGN \cite{rossi2020temporal} intends to refine the temporal signals of historical interactions by adding a memory module to the TGAT and achieves superior performance in the above DGNN models. To some extent, the improvement stems from the extraction of users' short-term and long-term preferences with the memory module. However, during aggregating history, TGN, which calculates the mean value or keeps the latest records in its memory module, is relatively preliminary. A further refinement of short-term history is necessary.

As described, aforementioned models like TGN are becoming an important tool to realize the proactive caching, but still leave questioned on how to design a much more effective information aggregation method in CDGNN. Therefore, we introduce an attention mechanism to the TGN model so as to capture both structural and
dynamical patterns. And inspired by the AoI in wireless sensor network (WSN) \cite{sun2017update}, a metric of information freshness to balance the huge data and the limited transmission capacity, we novelly introduce the concept into our neural network for selecting fresh messages adaptively.

\begin{table}[]
\renewcommand\arraystretch{1.2}
\normalsize
\centering
\caption{A Summary of Used Notations in This Paper.}
\label{n_table}
\scalebox{0.65}{%
\begin{tabular}{ll}
\hline
Symbol & Definition \\ \hline
$u_J, i_K$ & \begin{tabular}[c]{@{}l@{}}The notations of $j_{th}$ user and $k_{th}$ content, for $j\in[0,J]$ and $K\in[0, K]$;\\ $J$ is the max number of user, $K$ is the max number of content.\end{tabular} \\
$\textbf{\emph{v}}_{u_j}, \textbf{\emph{v}}_{i_k},\textbf{\emph{e}}_{jk}$ & The features of users$j$, content $k$ and their edge. \\
$\delta_p$ & The updating period of content. \\
$p_j^k(\delta_p)$ & \begin{tabular}[c]{@{}l@{}}The actual possibility of content $k$ requested by user $j$ during the \\ content updating period $\delta_p$.\end{tabular} \\
$\tilde{p}_j^k(\delta_p)$ & \begin{tabular}[c]{@{}l@{}}The possibility of content $k$ requested by user $j$ during the content \\ updating period $\delta_p$ that is calculated by our model.\end{tabular} \\
$A^k(\delta_p)$ & \begin{tabular}[c]{@{}l@{}}The total request behaviors about content $k$ for all users during the\\ updating period $\delta_p$.\end{tabular} \\
$\tilde{A}^k(\delta_p)$ & \begin{tabular}[c]{@{}l@{}}The total request behaviors about content $k$ for all users during the\\ updating period $\delta_p$ that we predicte.\end{tabular} \\
$P_I$ & The threshold value for determining whether a request will occur or not. \\
$\mathcal{C}(\delta_p)$ & The set of top-$C$ content that we predicte to be cached. \\
$\rm{F}(\cdot)$ & \begin{tabular}[c]{@{}l@{}}A multi-layer perceptron that is used to calculate the preference \\ between users and content.\end{tabular} \\
$\Delta_{jk}(t,n)$ & The age of user $j$'s $n$-th information. \\
$\Delta_{T_p}$ & \begin{tabular}[c]{@{}l@{}}The difference between the target timestamp $T_p$ to predict and the \\ most recent timestamp $T_l$ of a request in history\end{tabular} \\
$T_u$ & The period that we adopt to update memroy. \\ \hline
\end{tabular}%
}
\end{table}

\section{System Model and Problem Formulation}\label{sec3}
\subsection{System Model}
\subsubsection{Content Caching Model}
In this paper, we consider a network that contains some edge servers that provide in-network storage, e.g., the base stations (BSs). In each edge server, there are multiple users in its coverage area and these clients always request content aperiodically. To realize the goal of caching, we compute content popularity for the users within the edge server and download the most popular ones. The main purpose of our paper is to obtain a better hit rate with a well-performed popularity prediction model.

\begin{figure}[tbp]
\centering  
\includegraphics[scale = 0.475]{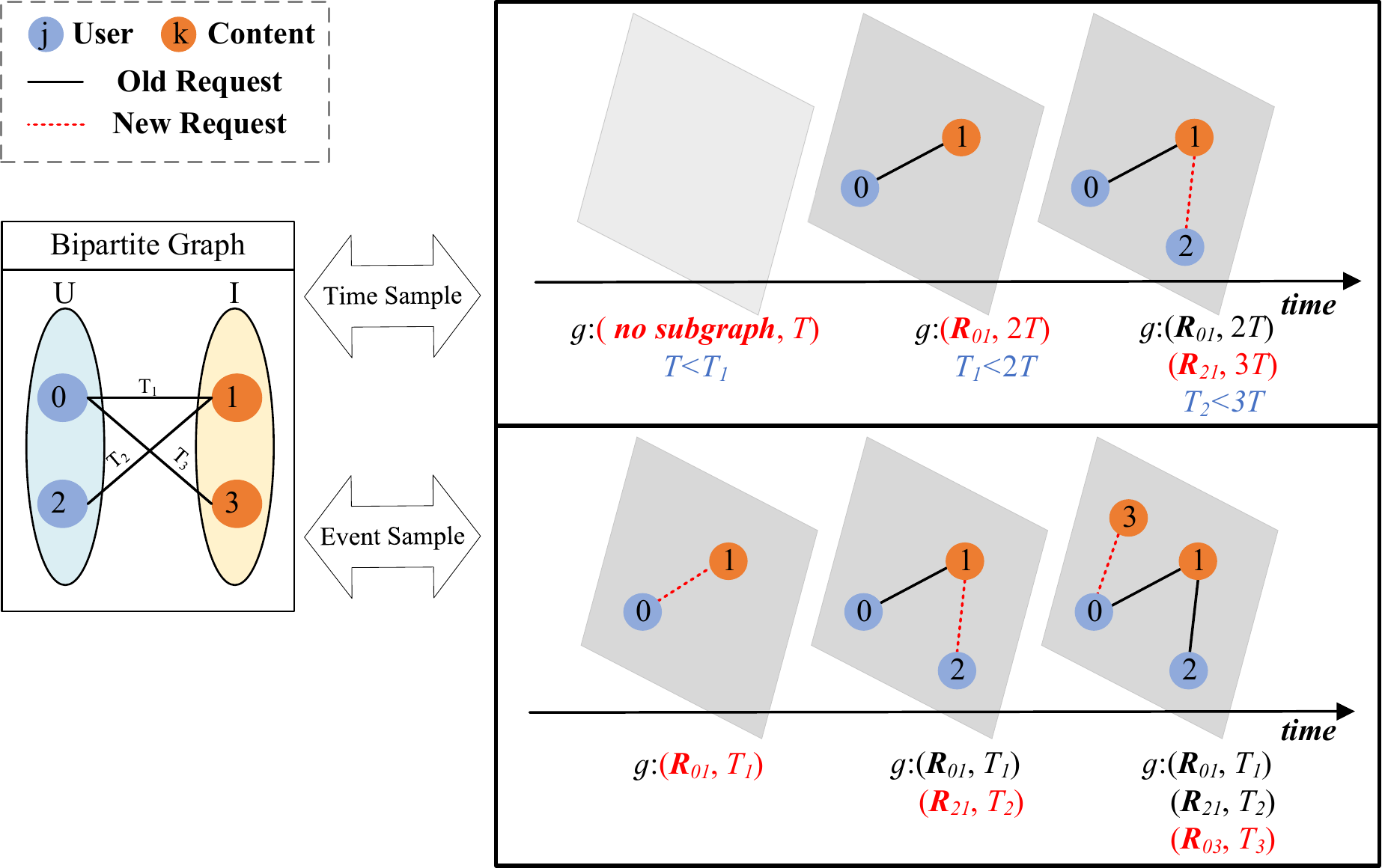}
\caption{Dynamic graph of requests with time sampling and event sampling.}  
\label{DG}
\end{figure} 

For such a user-content system with $J$ users and $K$ contents, the set of users can be indicated by $\mathcal{U} = \{u_0, u_1, \cdots ,u_J\}$ and the set of available contents in the network are identified as $\mathcal{I} = \{i_0, i_1, \cdots ,i_K\}$. For all the entities, we define their necessary raw information (e.g., the age of a user or the category of content) as the features: $\mathcal{V_U} = \{ \textbf{\emph{v}}_{u_0},\textbf{\emph{v}}_{u_1},\cdots ,\textbf{\emph{v}}_{u_J}\}$ and  $\mathcal{V_I} = \{\textbf{\emph{v}}_{i_0},\textbf{\emph{v}}_{i_1},\cdots ,\textbf{\emph{v}}_{i_K}\}$. $\textbf{\emph{v}}_{u_j}$  and $\textbf{\emph{v}}_{i_k}$, where $j \in 1\cdots J$ and $k \in 1\cdots K$, denote the original features of user $j$ and content $k$, respectively.

To accomplish the selection of future popular content, we define the possibility of content $k$ requested by user $j$ during the content updating period $\delta_P$ as $p_j^k(\delta_p)$. In this paper, we adopt DGNN model to calculate the possibility, as $\tilde{p}_j^k(\delta_p)$. Generally, since users at different time have different preferences towards the same content, there emerges an implicit evolving relationship between users and the requested content. To finish the task of caching, we need an action indicator function to indicate whether an item will be requested or not. The total request behaviors about content $k$ for all users $\mathcal{U}$ within the edge server during the updating period can be formed as:
\begin{equation}\label{eq0}
A^k(\delta_p, p_j^k) = \sum_{j \in \mathcal{U}}\textbf{1}(p_j^k(\delta_p) > P_I),\  \forall k \in\mathcal{I},
\end{equation}
where $P_I$ is the threshold value for requesting and any potential interaction with a possibility that is larger than $P_I$ will be regarded as the possible request and recorded with the function $\textbf{1}(\cdot)$. Moreover, we define the predicted request actions toward content $k$ during the updating period as $\tilde{A}^k(\delta_p, \tilde{p}^k)$.

Due to the limited storage space, we denote the maximum number of the edge server's capability as $C \geq 0$. After predicting and counting users' possible actions, we then make a popularity ranking list for content in $\mathcal{I}$ and rely on the set of top-$C$ items $\mathcal{C}(\delta_p)$ to update the cache space. Given the cached set $\mathcal{C}(\delta_p)$, the hit rate of our scheme during the content updating period $\delta_p$ can be denoted as $h(\delta_p)$, can be calculated as:
\begin{equation}\label{eq1.1}
h(\delta_p) = \frac{\sum_{k\in \mathcal{C}(\delta_p)}\tilde{A}^{k}(\delta_p, \tilde{p}_j^k)}{\sum_{k'\in \mathcal{I}}{A}^{k'}(\delta_p, p_j^k)}.
\end{equation}

\subsubsection{Dynamic Graph Model}
Note that the choice of $\mathcal{C}(\delta_p)$ is the important part for the increase of caching hit rate, and it has a positive relation with the possibility of request behavior that we need to estimate. Undoubtedly, an accurate prediction towards popularity is the key part of a caching strategy, but users and their interests are often subject to change over time, which increases its difficulty. To proactively cache content in advance, we can forecast the popularity based on the historical request information. Considering the existence of inactive users, it will be much more effective to learn from users' structural patterns. Thus, we can utilize the technique of GNN and view the whole system as an evolving graph.

We reformulate the aforementioned sets as the components of a bipartite graph $\mathcal{G}$, where users and content can be regarded as vertexes of the dynamic graph, and the interactions are naturally identified as the bipartite graph's evolving edges, $\mathcal{E} = \{\textbf{\emph{e}}_{01}, \textbf{\emph{e}}_{21},\textbf{\emph{e}}_{03},..., \textbf{\emph{e}}_{jk}\}$, where $\textbf{\emph{e}}_{jk}$ represents the vector of interaction between user $j$ and content $k$. As time goes by, correspondingly, new edges, as well as their participants, are added into the dynamic graph, as shown in the bottom of Fig. \ref{DG}. For the sake of simplicity, we don't consider the choice between multiple content providers for the same content and we also assume that all the different files have an equal size.

Furthermore, the occurrence of new requests can be deemed as the generation of edges, and for a dynamic graph, its timestamp should also be simultaneously recorded: $\mathcal{G} = \{(\textbf{\emph{R}}_{01},T_1),(\textbf{\emph{R}}_{21},T_2),(\textbf{\emph{R}}_{03},T_3),...,(\textbf{\emph{R}}_{jk},T_N)\}$, where $T_N$ denotes the occurrence timestamp of the edge, and $\textbf{\emph{R}}_{jk} = \{\textbf{\emph{v}}_{u_j}, \textbf{\emph{v}}_{i_k},\textbf{\emph{e}}_{jk}\}$. We use $\textbf{{Inf}}_{jk} = f(\textbf{\emph{v}}_{u_j},\textbf{\emph{v}}_{i_k},\textbf{\emph{e}}_{jk})\in\mathbb{R}^d$ to demonstrate an interaction event, where $f(\cdot)$ is the concatenation function, to denote the message between user $j$ and content $k$. Similarly, the initial embedding of message from the perspective of content $k$ is presented as $\textbf{{Inf}}_{kj}=f(\textbf{\emph{v}}_{i_k},\textbf{\emph{v}}_{u_j},\textbf{\emph{e}}_{jk})$. For node ${u_j}$ at time $t$, we denote the set of its request targets as: $\mathcal{N}(\textbf{\emph{v}}_{u_j};t) = \{\textbf{\emph{v}}_{i_0}, \textbf{\emph{v}}_{i_1}, ..., \textbf{\emph{v}}_{i_M}\}$. Subsequently, a DGNN model can be employed to generate the representations containing structural and dynamical patterns for further computation. Moreover, since the interactions we cope with are instantaneous, so we ignore the lasting time of an interaction and only focus on the occurrence of the request.

\subsection{Temporal Graph Network Model for Popularity Prediction}
We discuss how to predict the popularity after formulating the interaction between users and the requested content as a dynamic bipartite graph. Since the graph is not static, we refer to the temporal graph network \cite{rossi2020temporal} for discovering its underlying temporal relationship concurrently. The components of this model are defined as follows:

\begin{itemize}
\item The time-coding module, which is proposed in \cite{Xu2020Inductive}, is used to encode the timestamp of a request's built-up time. However, the interval between the latest request and the target one is much more meaningful than the absolute time points, so we use the difference $\Delta_t$ instead. To some extent, it achieves a preliminary extraction of temporal information. Consistent with TGAT, the function is defined as:
\begin{equation}\label{eq1}
\Phi _{d_T}(\Delta_t) = \sqrt{\frac{1}{d_T}}[\cos(\omega _1\Delta_t),\cos(\omega _2\Delta_t), ..., \cos(\omega _d\Delta_t)]^T
\end{equation}
where $\Phi _{d_T}(\Delta_t) \in \mathbb{R}^{d_T}$, $\omega_1, \omega_2, ...,\omega_d$ are the parameters to be trained and $d_T$ is the dimension number of the time embedding we want.

\item The time-concatenating module combines the initial message embedding $\textbf{{Inf}}_{jk}$ between users and content with the encoded time feature as the intact representation of an interaction, and the set of user $j$'s all interactions is described as:
\begin{equation}\label{eq2}
\textbf{{Msg}}_j(t) = [\textbf{{Inf}}_{j0}||\Phi _{d_T}(0),...  ,\textbf{{Inf}}_{jk}||\Phi _{d_T}(\Delta{t_N})]^T
\end{equation}
where the operator $||$ denotes a concatenation operation and $\textbf{{Msg}}_j(t)\in \mathbb{R}^{N\times(d+d_T)}$ is the final input feature to the DGNN model. 

\item Inspired by TGN \cite{rossi2020temporal}, a memory embedding module is also adopted to reinforce the dynamics refinement, i.e., the extraction of short-term and long-term interests of users. It aggregates the historical interactions stored in the message buffer to obtain an embedding with richer temporal information about short-term preference. Subsequently, it updates former memory with the embedding to inherit and renew the long-term one. Considering that the aggregation methods in TGN are elementary for extracting short-term temporal characteristic, we further propose to adopt an AoI-based attention mechanism for utilizing raw request information and we will talk about it later. 

\item Finally, an embedding module is adopted to extract the structural pattern within the graph and merge the time feature that we want to predict for further computation. The embedding module outputs the final representations, which contain both temporal and structural characteristics for all vertexes in the dynamic bipartite graph. Specifically, the final representation of user $j$ and content $k$ are denoted as $\textbf{{E}}^u_{j}$ and $\textbf{{E}}^i_{k}$. Based on these embedding representations, we can predict the preference of user $j$ for content $k$ as 
\begin{equation}
\label{corr}
    \tilde{p}_j^k(\delta_p) = {F}(\textbf{{E}}^u_{j}, \textbf{{E}}^i_{k})
\end{equation}
where a multi-layer perceptron (MLP) is a good choice to achieve the function ${F}(\cdot)$. Besides, we extract the same number of negative samples to speed up the training. Finally, we use a binary cross entropy (BCE) loss function for optimizing the whole neural network:
\begin{equation}
    \mathcal{L} = \frac{1}{n}\sum_{\gamma}-{(y_{\gamma}\log(\tilde{p}_{\gamma}) + (1 - y_{\gamma})\log(1 - \tilde{p}_{\gamma}))}
\end{equation}
where $y_{\gamma}$ is the label of the ${\gamma}$-th samples, $y_{\gamma}=1$ if the ${\gamma}$-th is a positive sample and $y_{\gamma}=0$ for negative samples. $\tilde{p}_{\gamma}=\tilde{p}_j^k$ is prediction result of the ${\gamma}$-th samples calculated by Eq. \eqref{corr}.
\end{itemize}

After estimating the user preference $\tilde{p}_j^k, \forall j,k$, we can derive the request behaviors $\tilde{A}^k(\delta_p, \tilde{p}_j^k)$ and the caching policy by Algorithm \ref{alg:algorithm1} and Algorithm \ref{alg:algorithm2}.

\section{AoI-based Temporal Graph Neural Network}\label{sec4}
Although the future requests of a user are diverse, comparatively accurate predictions can still be speculated by aggregating and analyzing existing history information. To some extent, aggregating history information can extract the short-term interests of users. But as described above, the ways of aggregation in TGN are still far from perfection: 1) the method of always keeping the latest request may not have much negative impact on the trained nodes, but it may lead to some errors when a rookie node joins into the graph due to the lack of historical references. 2) the method of averaging ignores the fact that requests in different time have distinct influence on future behavior, and some are obsolete that may bring adverse effects. Thus we propose an AoI-based attention mechanism to ameliorate the model. 

\begin{figure*}[tbp]
\centering  
\includegraphics[scale = 0.625]{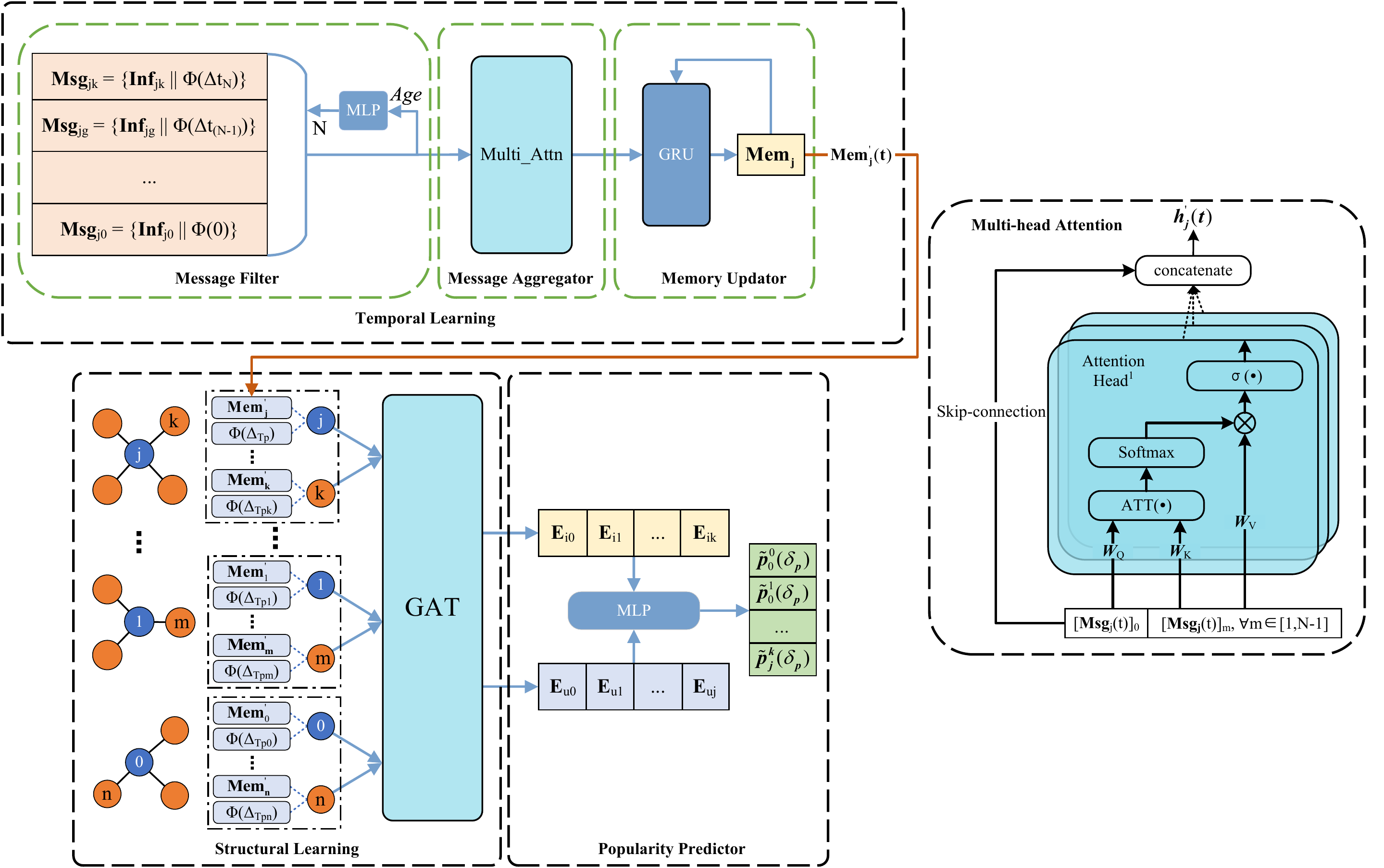}
\caption{The illustration of AoI-based temporal attention graph neural network and the multi-head attention mechanism.} 
\label{fig:tgn_re1}
\end{figure*} 

\subsection{Attention Mechanism for Temporal Pattern Extraction}\label{4.1}
As demonstrated in TGN, the more recent and repeated actions usually have greater impacts on the future interests prediction, we can calculate the degrees of correlation between previous interactions and the latest one with a self-attention mechanism \cite{vaswani2017attention}, which are used as the indications of influence later. Hence, we name this model TGN-A for short. Notably, the weighted summation of all the chosen messages is the aggregated feature to update the memory, as shown in Fig. \ref{fig:tgn_re1}. The aggregated message of user $j$ is calculated as follows:
\begin{equation}\label{eqa3}
\begin{aligned}
\textbf{\emph{h}}_j(t)&={ATT}(\textbf{{Msg}}_j(t),[\textbf{{Msg}}_j(t)]_0 ,\textbf{\emph{V}})\\
&=\sigma(\sum_{m\in N} \alpha_{jm}\textbf{\emph{V}}_{jm}) \in\mathbb{R}^{d_h}
\end{aligned}
\end{equation}
where $\textbf{\emph{V}}_{jn}= [\textbf{{Msg}}_j(t)]_n \textbf{\emph{W}}_{V}$, $\textbf{{W}}_{V}\in\mathbb{R}^{(d+d_T)\times d_h}$ is the weight matrix that needs to be trained. $\alpha_{jm}$ is the attention coefficient, indicating the importance between the $m$-th history message and the most recent one, and $\alpha_{jm}$ can be specifically presented by:
\begin{equation}\label{eq3.1}
\begin{aligned}
\alpha_{jm}&=\rm{softmax}([\textbf{{Msg}}_j(t)]_m, [\textbf{{Msg}}_j(t)]_0)\\
&=\frac{\exp\left(([\textbf{{Msg}}_j(t)]_0 \textbf{\emph{W}}_{Q})( [\textbf{{Msg}}_j(t)]_m \textbf{\emph{W}}_{K})\right)} {\sum_{n = 1}^{N} \exp\left(([\textbf{{Msg}}_j(t)]_0 \textbf{\emph{W}}_{Q}) ([\textbf{{Msg}}_j(t)]_n \textbf{\emph{W}}_{K})\right)} 
\end{aligned}
\end{equation}
where the $\textbf{{W}}_{Q}$, $\textbf{\emph{W}}_{K} \in\mathbb{R}^{(d+d_T)\times d_h}$ are the weights allocated to mix the request features with encoded time for producing the integrated features from different perspectives. Besides, $\rm{softmax}(\cdot)$ is a function that aims at normalizing the obtained coefficients and emphasizing the weights of important elements. $[\textbf{{Msg}}_j(t)]_m$ or $[\textbf{{Msg}}_j(t)]_n$ for $m, n = 1,..., N-1$ are the $m$-th or $n$-th message in $\textbf{{Msg}}_j(t)$ and $[\textbf{{Msg}}_j(t)]_0$ is the latest one. 

It is worth noting that, instead of considering the effect of all historical requests, inspired by GraphSAGE \cite{hamilton2017inductive} and TGN \cite{rossi2020temporal}, we only sample the most recent $N$ historical messages for aggregating. As for those only with $M$ interactions, where $M<N$, we pad their history behavior set $\textbf{{Msg}}_j(t)$ with mask operation, as done in Transformer \cite{vaswani2017attention}.

Besides, in view of the excessive smoothing effect incurred by increasing the number of network layers for information transmission and aggregation of nodes, we introduce the ``skip-connection'' derived from ResNet \cite{he2016deep} to combine the obtained representation with the base message $[\textbf{{Msg}}_j(t)]_0$ by a feed-forward network ${FN}(\cdot)$. We believe it will also improve the overall performance of the model by capturing non-linear interactions between the features:
\begin{equation}\label{eq4}
\begin{aligned}
\overline{\textbf{{Msg}}_j(t)}&=\textbf{\emph{h}}'_j(t) = {FN}\left(\textbf{\emph{h}}_j(t) ||  \textbf{\emph{R}}_{j0}\right) \\
&= {\rm ReLU}\left((\textbf{\emph{h}}_j(t) || \textbf{\emph{R}}_{j0})\textbf{\emph{W}}_0 + \textbf{\emph{b}}_0\right)\textbf{\emph{W}}_1+\textbf{\emph{b}}_1
\end{aligned}
\end{equation}
where $\textbf{\emph{h}}'_j(t) \in \mathbb{R}^{d}$ is the final aggregated vector representing the time-aware embedding at time $t$, and can also be denoted as $\overline{\textbf{{Msg}}_j(t)}$. And $\textbf{\emph{W}}_0 \in \mathbb{R}^{(d_h+d)\times d_0}$, $\textbf{\emph{W}}_1 \in \mathbb{R}^{d_0 \times d}$, $\textbf{\emph{b}}_0 \in \mathbb{R}^{d_0}$, $\textbf{\emph{b}}_1 \in \mathbb{R}^{d}$ are weight parameters to be trained in the neural network.

Empirically, \cite{vaswani2017attention} suggests that a multi-head attention may avoid the instability of the training process and promotes the performance of self-attention. We also extend the adopted attention to the multi-head setting, as presented in the right side of Fig. \ref{fig:tgn_re1}. We conduct an attention mechanism with $L$ heads and record each head's output for user $j$ and content $k$ with Eq. \eqref{eqa3} and Eq. \eqref{eq3.1} as: $\textbf{\emph{h}}^{(l)}_j(t)={ATT}^{(l)}\left([\textbf{{Msg}}_j(t)]^{(l)},[\textbf{{Msg}}_j(t)]_0^{(l)},\textbf{\emph{V}}^{(l)}\right), l=1,2,...,L$, and the weights values in different heads are various. Actually, it is a process that several single-head attentions carried out independently in parallel to gain deeper insights from different observation angles. Consequently, we modify Eq. \eqref{eq4} as:
\begin{equation}
\label{eq_mha}
\overline{\textbf{{Msg}}_j(t)}=\textbf{\emph{h}}'_j(t) = {FN}\left(\textbf{\emph{h}}^{(1)}_j(t) || \textbf{\emph{h}}^{(2)}_j(t)||...\textbf{\emph{h}}^{(L)}_j(t)|| \textbf{\emph{R}}_{j0}\right)
\end{equation}
In our experiment, we find the final attention architecture with three heads $(L = 3)$ will lead to a satisfactory result.

After extracting short-term interests with multi-head attention, we choose the Gated Recurrent Unit (GRU) \cite{cho2014learning} to finish the memory updating. The vector of new memory for user $j$, $\textbf{{Mem}}'_j$, derived from the memory updater where update the former memory $\textbf{{Mem}}_j$ with the aggregated message $\textbf{\emph{h}}_j'(t)$, and it contains all the history that we have chosen in aggregating module. To some extent, it also captures the long-term interest within the request data:
\begin{equation}
\label{gru}
\begin{aligned}
\textbf{{Mem}}_j' &= {GRU}\left(\textbf{{Mem}}_j, \overline{\textbf{{Msg}}_j(t)}\right)\\
&={GRU}\left(\textbf{{Mem}}_j, \textbf{\emph{h}}_j'(t)\right)\\
&=\textbf{\emph{Z}}_t\cdot\textbf{\emph{H}}_t +\left(1-\textbf{\emph{Z}}_t\right)\cdot \textbf{{Mem}}_j
\end{aligned}
\end{equation}
where $\textbf{\emph{Z}}_t$ is the update gate which decides the proportion of information that need to be inherited from the last hidden state, and $\textbf{\emph{H}}_t$ is the hidden state which is produced by ignoring some previous state $\textbf{{Mem}}_j$ and resetting the current input $\overline{\textbf{{Msg}}_j(t)}$ or $\textbf{\emph{h}}_j'(t)$ by a reset gate $\emph{F}_t$. They are conducted as follow:
\begin{equation}
\begin{aligned}
\textbf{\emph{Z}}_t=&\sigma\left(\textbf{\emph{h}}_j'(t)\textbf{\emph{W}}_{hZ}+\textbf{{Mem}}_j\textbf{\emph{W}}_{MZ}+\textbf{\emph{b}}_Z\right)\\
\textbf{\emph{F}}_t=&\sigma\left(\textbf{\emph{h}}_j'(t)\textbf{\emph{W}}_{hF}+\textbf{{Mem}}_j\textbf{\emph{W}}_{MF}+\textbf{\emph{b}}_F\right)\\
\textbf{\emph{H}}_t=\rm{tanh}&\left(\textbf{\emph{h}}_j'(t)\textbf{\emph{W}}_{hH}+\left(\textbf{\emph{F}}_t\cdot\textbf{{Mem}}_j\right)\textbf{\emph{W}}_{MH}+\textbf{\emph{b}}_H\right)
\end{aligned}
\end{equation}
where $\textbf{\emph{W}}_{hZ}$, $\textbf{\emph{W}}_{hF}$, $\textbf{\emph{W}}_{hH}$, $\textbf{\emph{W}}_{MZ}$, $\textbf{\emph{W}}_{MF}$, $\textbf{\emph{W}}_{MH}$ are the weights of the recurrent neural networks, $\textbf{\emph{b}}_Z$, $\textbf{\emph{b}}_F$, $\textbf{\emph{b}}_H$ are their bias values. The activation function $\sigma(\cdot)$ is used to limit the result within 0 and 1. Due to the process of forgetting and updating, we obtain a new feature that contains user's long-term interests as well as the short-term ones. Similarly, a Long Short-Term Memory (LSTM) or other RNN architectures have been proved that they also have a similar gain in \cite{Xu2020Inductive}.

\subsection{AoI-based Attention Mechanism}\label{4.2}

Faced with massive historical information, one of the essential issues lies on deciding the amount of information that deserves to be aggregated. If the chosen requests occurred too long before, it may have an negative impact on the future, let alone using all precious requests, leading to an unnecessary increasement in computing cost. Even we have restricted the size of aggregated neighborhoods, the selected information may still be non-positive to our prediction. To optimize the result, we introduce the concept of AoI, which is a commonly-adopted metric of information freshness in WSN. Inspired by the application of filtering fresh data in WSN, we tend to pass AoI to the TGN with attention (TGN-A) model for excluding the information with an age that is too stale to be detrimental to the final prediction. For simplicity of representation, we denote the TGN with AoI-based attention for temporal learning as ATAGNN.

As introduced in \cite{kosta2017age}, the age of information is usually defined as:
\begin{equation}
\label{aoi}
   \Delta_{jk}(t,n)= t-{B}_{jk}(n)
\end{equation}
where ${B}_{jk}(n)$ represents the birth time of the $n$th request between user $j$ and content $k$, and $t$ denotes the current time or the newest time point of the messages. Hence, we can easily figure out an age for each request. In fact, since our training is based on batches, we regard the latest moment in a batch as the independent variable. But unlike the serration distribution in traditional AoI computation, we regard the requests between two nodes at various time points as different interactions to avoid the influence reductions of repeated actions in the follow-up calculation.

We concatenate all the age data as an $N$-length vector $a_j$, and then pass it to a two-layer MLP module to finish the adaptive selecting task of the worthy information for each vertex. 
\begin{equation}\label{eqther}
\begin{aligned}
\textbf{\emph{h}}_{th} &= {MLP}(\emph{a}_j)=\rm{ReLu}(\emph{a}_j\textbf{\emph{W}}^1 + \textbf{\emph{b}}^1)\\
{\rm thre}_t &= {MLP}(\textbf{\emph{h}}_{th})=\rm{ReLu}(\textbf{\emph{h}}_{th}\textbf{\emph{W}}^2 + \textbf{\emph{b}}^2)
\end{aligned}
\end{equation}
where $a_j \in \mathbb{R}^{N}$, $\textbf{\emph{h}}_{th} \in \mathbb{R}^{d_{th}}$ and ${\rm thre}_t \in \mathbb{R}$. Besides, $\textbf{\emph{W}}^1\in\mathbb{R}^{N\times d_{th}}$ and $\textbf{\emph{b}}^1\in\mathbb{R}^{d_{th}}$, $\textbf{\emph{W}}^2\in\mathbb{R}^{d_th}$ and $\textbf{\emph{b}}^2\in\mathbb{R}^{1}$ are the training parameters of the 2-layer perceptron. We believe this module can simply fitting the general behavior based on the age of requests and offer a threshold for determining whether the information deserves to be taken into consideration or not. And we discuss the details of the feasibility in Appendix \ref{APP:a}.

However, the fitting ability of an MLP is relatively preliminary, which will bring certain deviations inevitably. In order to decrease the potential impact of such deviations, we also adopt a ``soft method” to promote the performance:
\begin{equation}
\label{eq5}
{t'^n}_{jk} = t^n_{jk} \cdot \sigma \left(100\ast(\Delta_{jk}(t,n)-\rm{thre}_t)\right)
\end{equation}
where $t^n_{jk}$ is the raw timestamps of the $n$th interaction between vertexes $j$ and $k$ in the dynamic graph $\mathcal{G}$, and ${t^n}'_{jk}$ is the one we actually use in subsequent calculation, ${\rm thre}_t$ is the threshold we obtain from the MLP. After calculating Eq. \eqref{eq5}, we mask all the requests with an age that is lower than the threshold time, while redistributing a little greater timestamp to those that are close to the threshold. Then we execute the multi-head attention module with the new set of valuable messages, as mentioned before. 

\subsection{Future and Structural Patterns Embedding}
As the final module in our model, we choose a GAT to accomplish the structure's deeper extraction and the generation of unique embedding representations for the participants of an interaction we want to predict. In order to map the future information, rather than the commonly adopted $\rm{LeakyReLu}(\cdot)$ in GAT \cite{velivckovic2017graph}, we promote the learning performance by adding the encoded target timestamps into the memory vectors and adopting a linear transformation with a dot-product:
\begin{equation}\label{newgat}
\alpha^g_{jk}= \frac{\exp((\tilde{\textbf{\emph{h}}}_k\textbf{\emph{W}}^g_Q)^T(\tilde{\textbf{\emph{h}}}_j\textbf{\emph{W}}^g_K)} {\sum_{m \in \mathcal{N}(v_j;t)} \exp((\tilde{\textbf{\emph{h}}}_m\textbf{\emph{W}}^g_Q)^T(\tilde{\textbf{\emph{h}}}_j\textbf{\emph{W}}^g_K)} 
\end{equation}
where $\textbf{\emph{W}}^g_K$ and $\textbf{\emph{W}}^g_Q$ are the weight parameters that we employ to capture the relationship between time encoding and the output of temporal leaning $\textbf{{Mem}}'_j$. Besides, the $\tilde{\textbf{\emph{h}}}_j = [\textbf{{Mem}}'_j || \Phi _{d_T}(\Delta_{T_p})]$, and $\Delta_{T_p} = T_p - T_l$ is the difference between the target timestamp $T_p$ to predict and the most recent timestamp $T_l$ of a request in history from user $j$, where $T_p \geq T_l$. As we presented before, the more attention heads, the better structural representations will be extracted. Therefore, we also encapsulate a multi-head attention mechanism into this module.

We desire to generate reliable representations for the potential users and content by the above model, and calculate their correlation degree as a reference of judging users' coming behavior. Finally, we summarize the above algorithm in Algorithm \ref{alg:algorithm1}. Meanwhile, we provide a comprehensive graphical illustration of our AoI-based temporal attention module for user $j$'s temporal learning in the upper left part of Fig. \ref{fig:tgn_re1}.

\begin{algorithm}[t]
\renewcommand{\algorithmicrequire}{\textbf{Input:}}
\renewcommand{\algorithmicensure}{\textbf{Output:}}
\caption{The preference prediction algorithm with AoI-based temporal attention GNN}
\label{alg:algorithm1}
\begin{algorithmic}[1]
\REQUIRE Dynamic request dataset;
\ENSURE The presentations of users $\textbf{\emph{E}}^u_{j}$ and content $\textbf{\emph{E}}^i_{k}$. And predicate the preference between $u_j$ and $i_k$ $p$\
\STATE {Initialize the parameters for the whole network;}\
\STATE {Initialize the memory buffer with zeros and message buffer.} \
\STATE {Restore the graph information ($\textbf{{Inf}} \gets \rm{all\ messages}$) and divide it into several mini batches;}\
\FOR {each $\rm{batch}(\textbf{\emph{v}}_{u_j}, \textbf{\emph{v}}_{i_k}, \textbf{\emph{e}}_{ui}, t) \in training\  dataset$} 
\STATE {$\dot{n}\gets \rm{Sample\ negatives}$;} 
\STATE {Calculate the age $\Delta_{jk}(t,n)$ and the threshold ${\rm thre}_t$;} 
\STATE {Filter and concatenate the valuable messages $\textbf{{Msg}}(t)$;}
\STATE {Aggregate with multi-head attention mechanism in Eq. \eqref{eqa3}, \eqref{eq_mha} and obtain $\overline{\textbf{{Msg}}_j(t)}$;}
\STATE {Update features $\textbf{{Mem}}_j$ in memory buffer with GRU in Eq. \eqref{gru};}
\STATE {Encode the time difference $\Delta_{T_p}$ with Eq. \eqref{eq1} for all nodes;}
\STATE {Concatenate the encrypted time feature with $\textbf{{Mem}}_j$, and obtained the new feature $\tilde{\textbf{\emph{h}}}_j$ as the input;}
\STATE {Obtain $\textbf{\emph{E}}^u_{j}(T_p)$ and $\textbf{\emph{E}}^i_{k}(T_p)$ through the modified $GAT$;}
\STATE {Predict the correlation degree between users and content with Eq. \eqref{corr};}
\STATE {Optimize this network with $\rm{BCELoss}(\cdot)$;}
\ENDFOR
\end{algorithmic}
\end{algorithm}

\begin{table*}[htbp]
\renewcommand\arraystretch{1.2}
\centering
\caption{The results of transductive learning and inductive learning tasks for predicting future links of nodes. TGN-L, TGN-M and TGN-A are the TGN model with keeping latest message, averaging all messages and using an attention mechanism for a fixed number of messages, respectively. 5n, 10n and 15n mean the number of neighbors that the attention mechanism aggregates. The best results in each column are highlighted in \textbf{bold} font and the best results in baseline are highlighted in \underline{underline}.} 
\label{table_1}
\resizebox{\textwidth}{!}{
\begin{tabular}{cc|cccc|cccc}
\hline
\multicolumn{2}{c|}{Dataset} & \multicolumn{4}{c|}{Wikipedia} & \multicolumn{4}{c}{MOOC} \\
\multicolumn{2}{c|}{Metric} & Old AUC & Old AP & New AUC & New AP & Old AUC & Old AP & New AUC & New AP \\ \hline
\multicolumn{1}{c|}{\multirow{5}{*}{Baseline}} & RNN & 76.476 & 77.100 & - & - & 70.394 & 70.461 & - & - \\
\multicolumn{1}{c|}{} & DyRep & 93.823 & 94.313 & 91.795 & 92.668 & 87.498 & 83.591 & 87.201 & 83.493 \\
\multicolumn{1}{c|}{} & TGAT & 95.391 & 95.710 & 93.241 & 93.810 & 74.413 & 69.932 & 73.211 & 69.282 \\
\multicolumn{1}{c|}{} & TGN-L & {\ul 98.406} & {\ul 98.470} & 97.700 & 97.805 & {\ul 92.026} & {\ul 89.855} & {\ul 92.447} & {\ul 90.494} \\
\multicolumn{1}{c|}{} & TGN-M & 98.342 & 98.426 & {\ul 97.741} & {\ul 97.864} & 90.951 & 88.423 & 92.271 & 90.319 \\ \hline
\multicolumn{1}{c|}{\multirow{3}{*}{TGN-A}} & 5n & 98.481 & 98.559 & 97.806 & 97.911 & 93.091 & 90.944 & 92.994 & 91.014 \\
\multicolumn{1}{c|}{} & 10n & 98.509 & 98.600 & 97.923 & 98.044 & 93.088 & 91.078 & 92.737 & 90.783 \\
\multicolumn{1}{c|}{} & 15n & 98.508 & 98.589 & 97.909 & 98.029 & 93.384 & 91.280 & 93.103 & 91.090 \\ \hline
\multicolumn{1}{c|}{\multirow{3}{*}{\begin{tabular}[c]{@{}c@{}}ATAGNN\\ without\\ Eq. \eqref{eq5}\end{tabular}}} & 5n & 98.514 & 98.591 & 97.925 & 98.035 & 93.183 & 91.150 & 93.109 & 91.111 \\
\multicolumn{1}{c|}{} & 10n & 98.530 & 98.605 & 97.883 & 98.013 & 93.554 & 91.603 & 93.310 & 91.421 \\
\multicolumn{1}{c|}{} & 15n & 98.510 & 98.594 & 97.955 & 98.074 & 93.541 & 91.516 & \textbf{93.424} & \textbf{91.478} \\ \hline
\multicolumn{1}{c|}{\multirow{3}{*}{\begin{tabular}[c]{@{}c@{}}ATAGNN\\ with\\ Eq. \eqref{eq5}\end{tabular}}} & 5n & \textbf{98.544} & 98.625 & 97.915 & 98.026 & 93.362 & 91.396 & 93.302 & 91.394 \\
\multicolumn{1}{c|}{} & 10n & 98.526 & 98.610 & 97.941 & \textbf{98.066} & \textbf{93.577} & \textbf{91.635} & 93.330 & 91.376 \\
\multicolumn{1}{c|}{} & 15n & 98.539 & \textbf{98.632} & \textbf{97.957} & 98.061 & 93.568 & 91.572 & 93.269 & 91.300 \\ \hline
\end{tabular}
}
\end{table*}

\section{Simulation Results and Numerical Analysis}\label{sec5}
In this section, we evaluate our performance of the models we mentioned above based on two real-world datasets: Wikipedia and MOOC. We also make a comparison between our models and four state-of-the-art methods designed for representation learning in temporal networks, including RNN \cite{DBLP:journals/corr/HidasiKBT15}, DyRep \cite{trivedi2019dyrep}, TGAT \cite{Xu2020Inductive} and TGN \cite{rossi2020temporal}. Besides, an experiment of cache hit rate in ICN between our prediction-based caching approach and the traditional policies (e.g., LRU and LFU) is also conducted.

\subsection{Dataset Description}\label{5A}
\textbf{Wikipedia Dataset:}
It is a public dataset that records the Wikipedia pages edited by users on Wikipedia within 30 days. The number of entries and users involved is 9227. There are more than 15000 interactions, which represent the number of edges in the bipartite graph. Besides, their interactions are time-stamped. Besides, we perform a 70\%-15\%-15\% chronological split for training, validation and testing.

\textbf{MOOC Dataset:}
It is also a public dataset that records the history of actions done by students on a MOOC online course within one month, e.g., watching a video, submitting an answer. We select 5763 users and 56 contents as the nodes of the dynamic graph, which also consists of 175,856 time-stamped interactions. Due to the relatively large amount of interactions, we perform a 60\%-20\%-20\% chronological split for training, validation and testing.

\begin{figure}[tbp] 

\centering  
\subfloat[Transductive Task]{
\centering  
\includegraphics[width=2.7in]{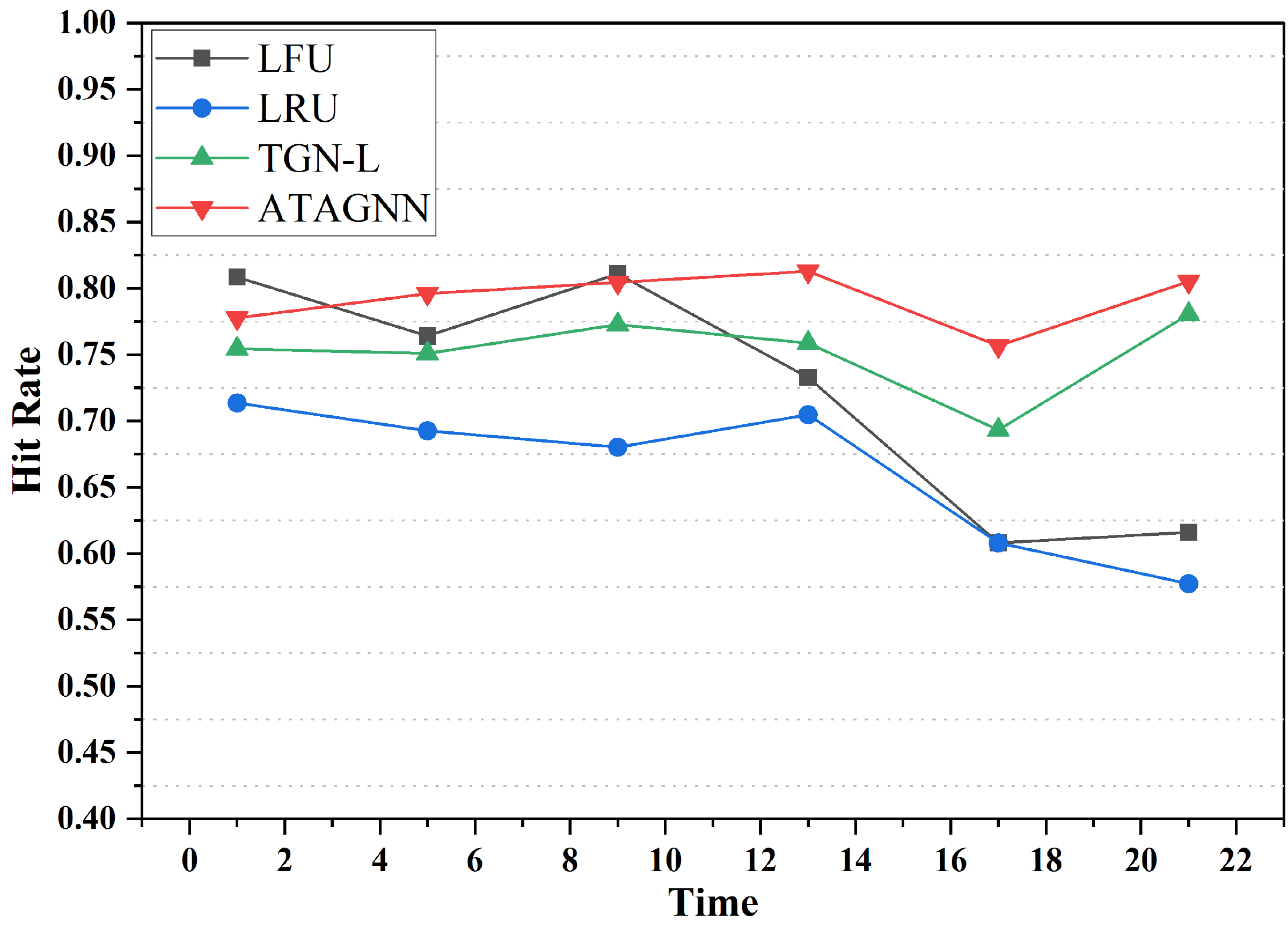}
}

\subfloat[Inductive Task]{
\centering  
\includegraphics[width=2.7in]{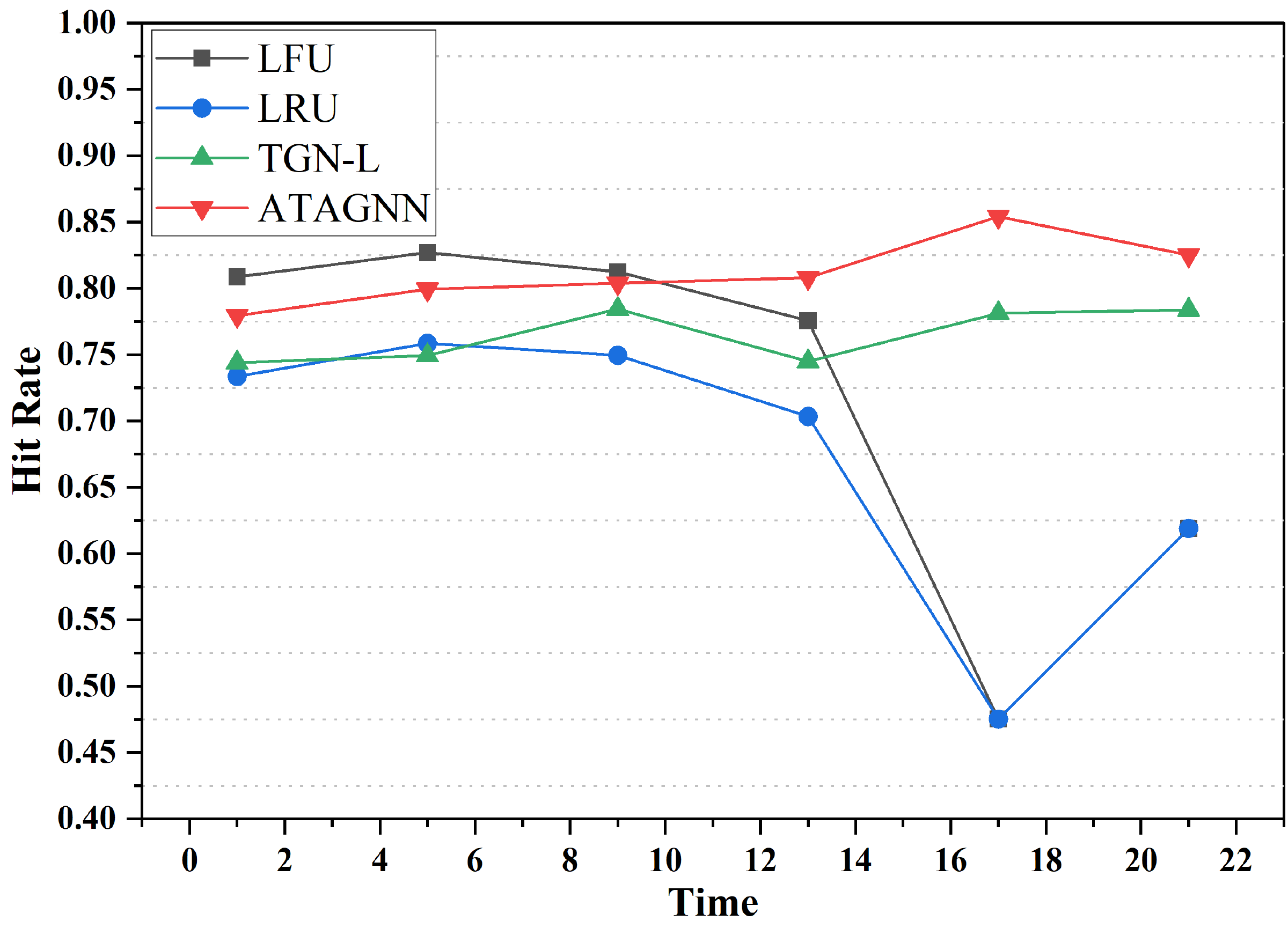}
}

\caption{24-hour hit rate performance of MOOC datasets with different algorithms in Transductive and inductive tasks.}  
\label{fig_re2}
\end{figure}

\subsection{Experimental Setup}\label{5B}

\textbf{Evaluation Tasks and Training Configuration:}
One of the superiority of our model is that it can be easily generalized to the new node. Thus, we verify our model's performance in two types of tasks, i.e., transductive task and inductive task. In the transductive task, we evaluate model's ability of predicting the temporal links for those nodes that have been observed in the training phase. In contrast, our target in inductive task is to inspect the model's talent in representing the nodes that have never been trained. Equally, all the chosen baselines based on GNN also carry out both tasks, while the RNN model is only able to accomplish the transductive one. We also set a max number of aggregating for TGN-A and our ATAGNN model. To demonstrate the influence of AoI, we have a comparison between our model and TGN-A with a fixed age threshold (i.e., 60s, 600s, 3600s.) and the max number of neighbors that we want to aggregate is 5.

\begin{table}[t]
\renewcommand\arraystretch{1.2}
\centering
\caption{The result of TGN-A with fixed age threshold (i.e., 60s, 600s, 3600s), the origin model without considering the age and the ATAGNN model with MLP for adaptively choosing threshold. The best results in each column are highlighted in \textbf{bold} font}
\label{tab:my-table}
\scalebox{0.9}{%
\begin{tabular}{cc|cccc}
\hline
\multicolumn{2}{c|}{Datasets} & \multicolumn{4}{c}{MOOC} \\
\multicolumn{2}{c|}{Metric} & Old AUC & Old AP & New AUC & New AP \\ \hline
\multicolumn{1}{c|}{\multirow{3}{*}{\begin{tabular}[c]{@{}c@{}}TGN-A\\ with\\ Age\end{tabular}}} & 60 & 92.363 & 90.191 & 92.234 & 90.004 \\
\multicolumn{1}{c|}{} & 600 & 92.695 & 90.612 & 92.686 & 90.696 \\
\multicolumn{1}{c|}{} & 3600 & 93.088 & 91.026 & 92.792 & 90.641 \\ \hline
\multicolumn{2}{c|}{TGN-A without Age} & 93.091 & 90.944 & 92.994 & 91.014 \\ \hline
\multicolumn{2}{c|}{ATAGNN} & \textbf{93.362} & \textbf{91.396} & \textbf{93.302} & \textbf{91.394} \\ \hline
\end{tabular}%
}
\end{table}

\begin{algorithm}[t]
\caption{The caching algorithm with neural network}
\label{alg:algorithm2}
\begin{algorithmic}[1]
\STATE {Initialize and load the whole neural network and the memory buffer;}\
\FOR {each $\rm{hour} \in [1, 24]$}
\IF{$\rm{hour} \% T_u == 0$}
    \FOR {$T_p = 0$; $T_p < 3600$; $T_p + \delta_p$}
    \STATE {Load the trained memory;}
    \STATE {Obtain $\textbf{\emph{E}}^u_{j}(\delta_p)$ and $\textbf{\emph{E}}^i_{k}(\delta_p)$ with the neural network;}
    \STATE {Obtain the correlation degrees $\tilde{p}_j^k(\delta_p)$ between users and content;}
    \STATE {Count the possible activity with Eq. \eqref{eq0};}
    \STATE {Clear the memory buffer;}
    \ENDFOR
\ELSE
    \STATE {Update the trained memory with real data;}
\ENDIF
\STATE {Add up all the $\tilde{A}^k(\delta_p)$ and sort them with a list;}
\STATE {Cache the top-$C$ content as according to the cache space. }
\ENDFOR
\end{algorithmic}
\end{algorithm}

Before the training, we sample an equal amount of negative interactions to the positive node pairs, regarding the prediction as a binary classifications problem, thus a BCELoss function is chosen. We adopt an Adam optimizer with a learning rate of 0.0001, a batch size of 200 for all examinations. Finally, we adopt \textit{Area Under the ROC Curve} (AUC) and \textit{Average Precision} (AP) \cite{skardinga2021foundations} as the metrics to indicate the performance.

\begin{figure*}[!t]
\subfloat[Transductive Task in Wikipedia]{
\centering
\begin{minipage}{0.4\textwidth}
\includegraphics[scale = 0.32]{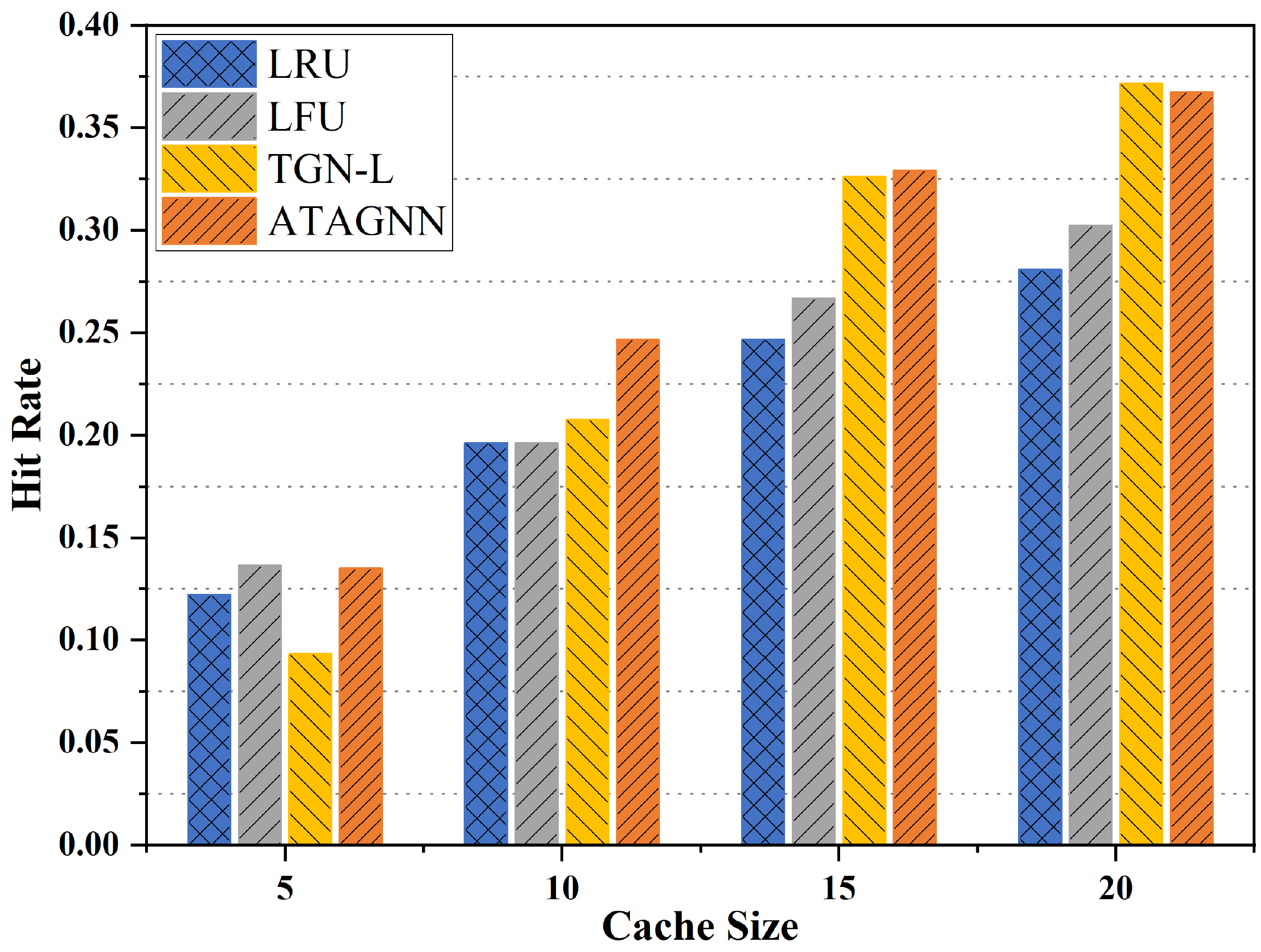}
\end{minipage}
}
\hfil
\subfloat[Inductive Task in Wikipedia]{
\centering
\begin{minipage}{0.4\textwidth}
\includegraphics[scale = 0.32]{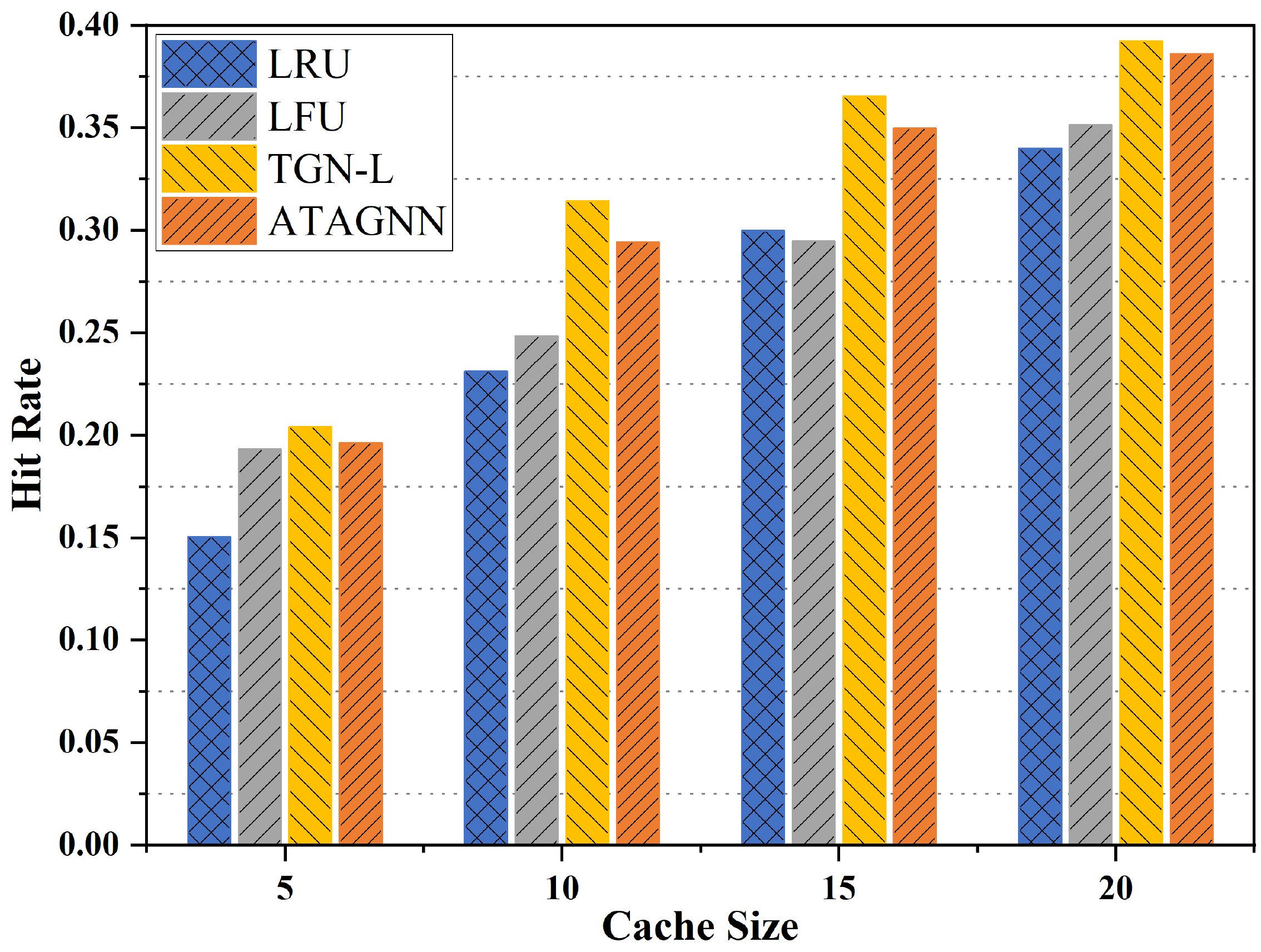}
\end{minipage}
}

\subfloat[Transductive Task in MOOC]{
\centering
\begin{minipage}{0.4\textwidth}
\includegraphics[scale = 0.32]{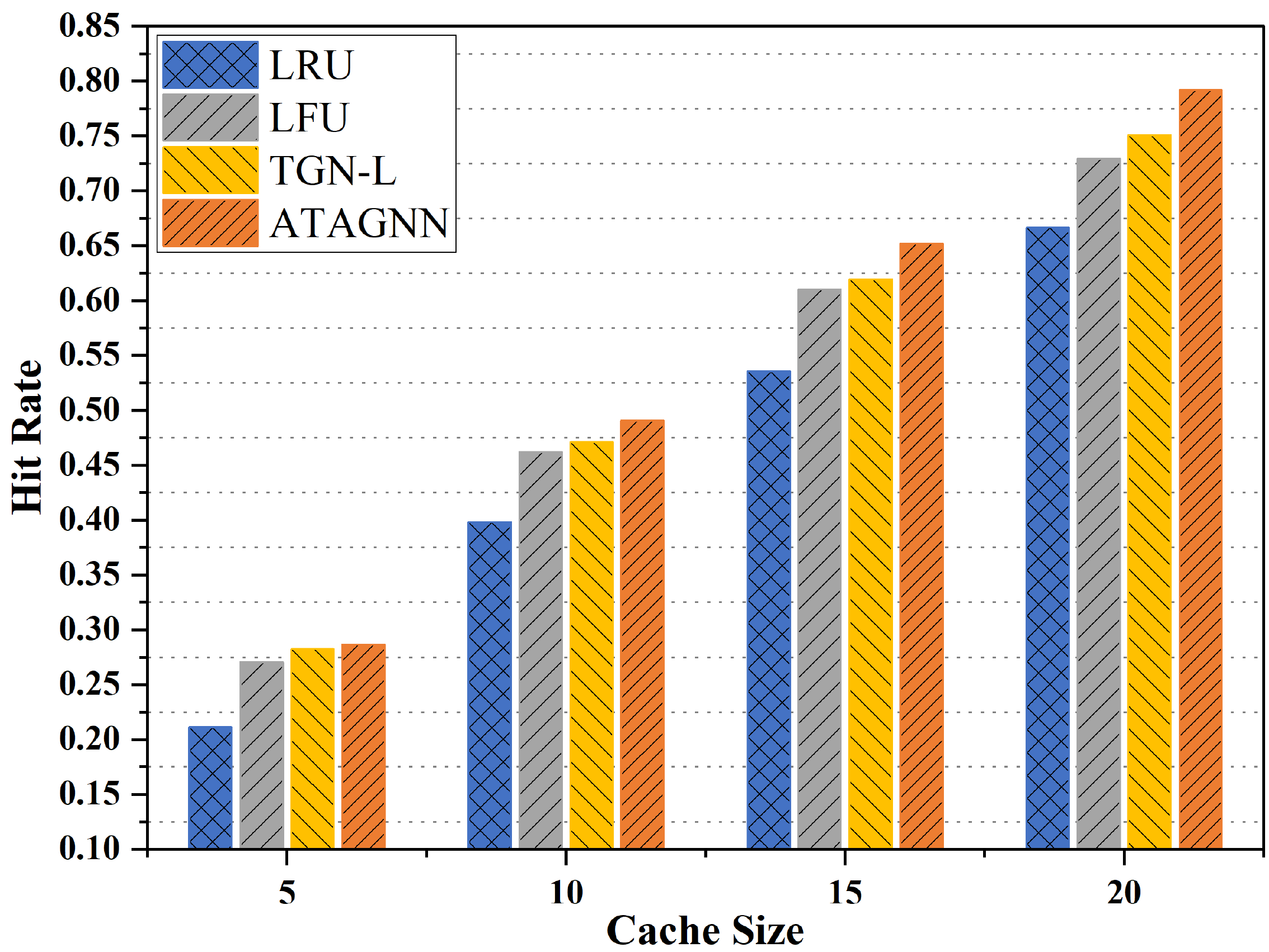}
\end{minipage}
}
\hfil
\subfloat[Inductive Task in MOOC]{
\centering
\begin{minipage}{0.4\textwidth}
\includegraphics[scale = 0.32]{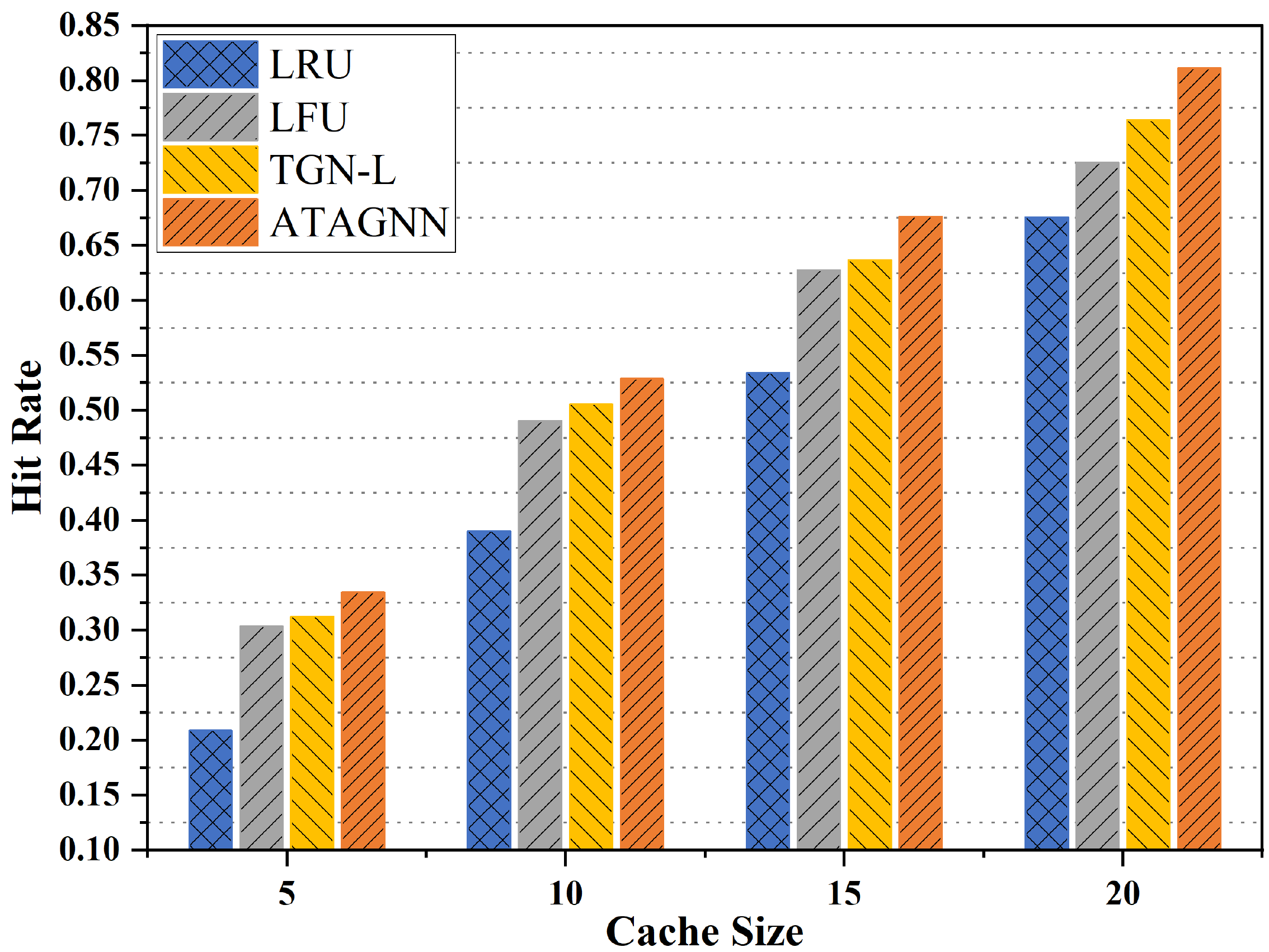}
\end{minipage}
}
\caption{Hit rate comparison with different cache sizes in different datasets.}  
\label{fig_re3}
\end{figure*}

\begin{figure*}[!t]
\subfloat[Transductive Task]{
\begin{minipage}{0.4\textwidth}
\includegraphics[scale = 0.345]{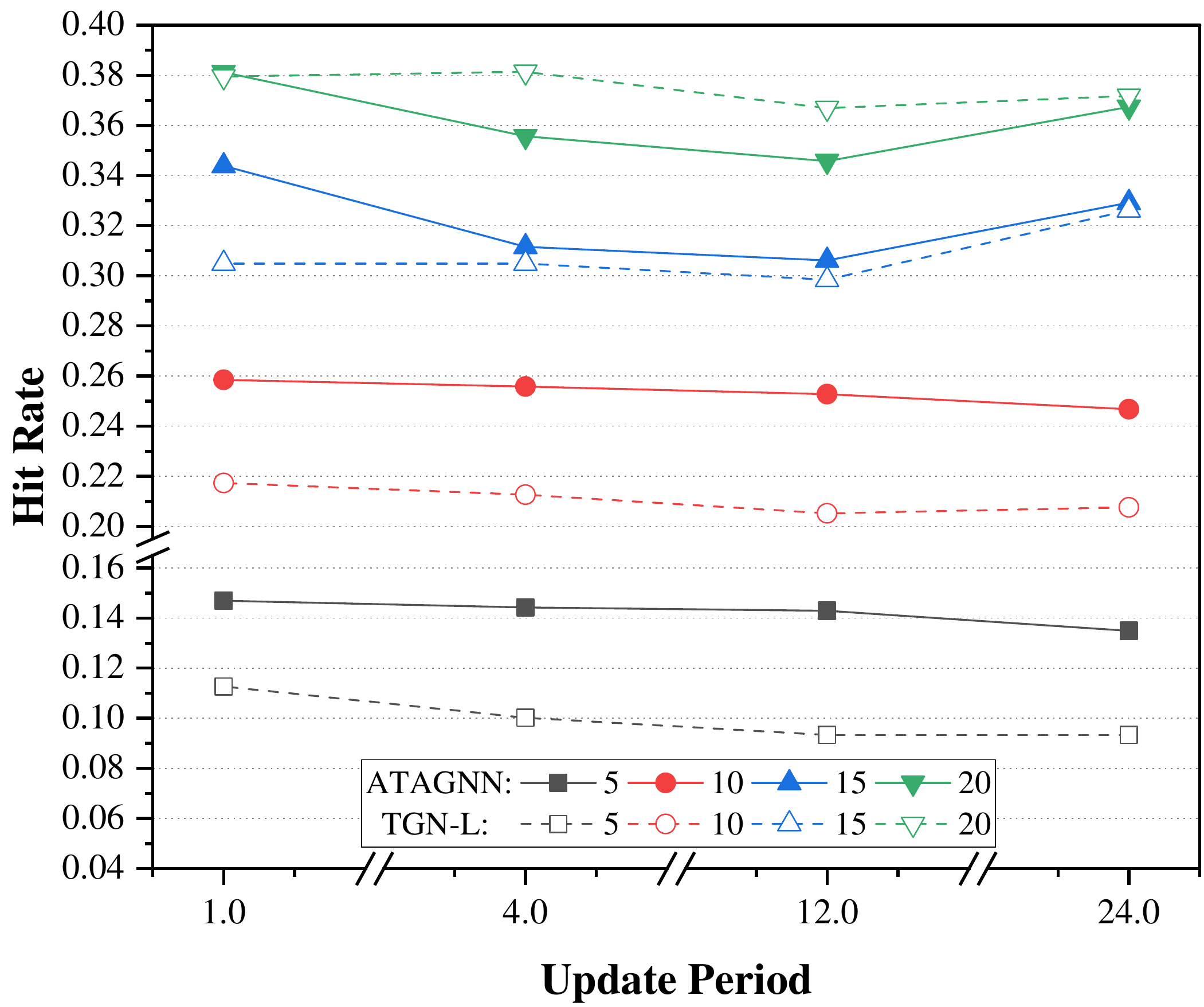}
\end{minipage}
}
\hfil
\subfloat[Inductive Task]{
\begin{minipage}{0.4\textwidth}
\includegraphics[scale = 0.345]{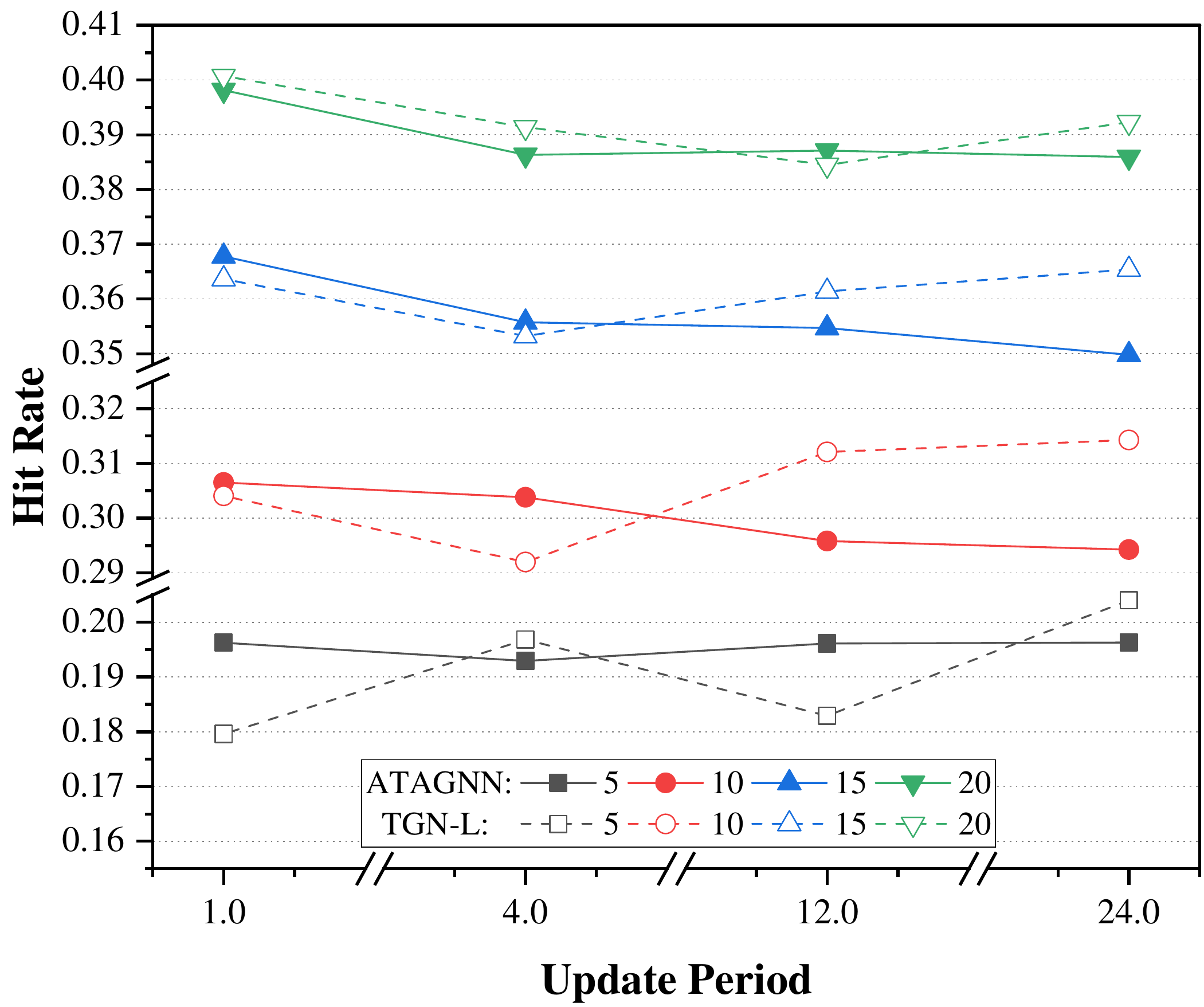}
\end{minipage}
}
\caption{Hit rate comparison with different update periods in Wikipedia.}  
\label{fig_re4}
\end{figure*}

\textbf{Caching Policy Setting:}
When we deploy our DGNN model to produce the caching policy, unlike in the validation and test exams, we have little prior knowledge about the users, content and the possible timestamps of events. Inspired by LFU, we narrow the prediction scope by only choosing those entities that have been observed in last one or two hours when the number of candidates is too large. We also design a caching algorithm that relies entirely on the results from the aforementioned model. The main idea of our policy is shown in Algorithm \ref{alg:algorithm2}. It caches the contents from all candidate items by counting the possible accessible actions with Eq. \eqref{eq0} as well as the results of our model and generating ``fake requests" to download the popular content in advance. Moreover, if we update the memory with those fake requests, they may mislead our subsequent prediction. Thus, we need to update the memory with real interactions periodically, and call it the memory update period, $T_u$.

We compare our popularity-based scheme with LRU and LFU. LRU updates the caching by replacing the content that has not been requested for a longest time, while LFU always tries to keep those that have been most requested. Moreover, to show the superiority of our ATAGNN over other models in caching, we also compare it with the algorithm based on the best state-of-the-art model, TGN-L.

\textbf{Cache-Hit Efficiency:} Due to the lack of information for predicting with LRU and LFU in the first hour, the comparison with the traditional schemes is completed based on the testing data of MOOC and Wikipedia within 23 hours and the cache size is set as 15. There are 55 content in MOOC and more than 500 items in Wikipedia with around 2000 users are involved in our caching test.

Besides, the results of hit rate with various caching spaces are evaluated as well. On the other hand, in all the aforementioned tests, the default models' updating period for memory buffer is 24 hours, i.e., it will generate all results at once and never update the model's memory in our simulation. Therefore we also carry out ablation studies to demonstrate the effectiveness of different memory update periods (e.g., $T_u = 24, 12, 4, 1$).

Moreover, we calculate users' preference to the content every 6-second with flexible thresholds in different datasets and list the ranking of the popularity for each hour. The tests also include transductive and inductive tasks. 

\subsection{Results Analysis}\label{5C}
\textbf{Prediction result:}
Table \ref{table_1} presents the prediction accuracy. It can be observed that compared with some latest models based on DGNN or RNN, an attention mechanism based TGN model can provide deeper insights into the temporal information, which leads to better results in both inductive and transductive tasks, even without AoI. Especially in the MOOC dataset, the average precision will reach an improvement of about 2\% after being integrated with an attention mechanism to aggregate historical interactions. Meanwhile, Table \ref{tab:my-table} demonstrates the positive impact of AoI, when the age threshold is 3600s, and clearly presents that our adaptive model can still win superior performance.

The results of TGN-L and TGN-M as well as our models shown in Table \ref{table_1} also prove that the number of aggregated information is an important factor that affects the results. Besides, it is apparent that after taking the AoI as a reference for choosing information, even without the Eq. \eqref{eq5}, our model is able to reach further improvement, since it effectively excludes some stale information for certain vertexes. Furthermore, due to the introduction of Eq. \eqref{eq5}, the model can slightly alleviate the error caused by the calculation of the threshold in the two-layer MLP, especially in transductive tasks.

Fig. \ref{fig_re2} and Fig. \ref{fig_re3} investigate the caching performance sensitiveness of our method. Fig. \ref{fig_re2} reveals the evolution of the caching strategy based on our model and the best state-of-the-art algorithm TGN-L as well as the traditional caching algorithms' hit rate within 23 hours with a cache space of 15. Thanks to the sufficient aggregation of historical information, ATAGNN based caching is able to keep superior for a long time, especially in Wikipedia. Besides, the smooth performance within the MOOC dataset also implies the stability of our models over traditional methods in caching. Fig. \ref{fig_re3} shows the average hit rate of our caching policy with different maximum cache sizes (i.e., 5, 10, 15, 20). It can be observed that our model can always provide a high-confidence prediction result of the popularity for caching and greatly improve the hit rate in most cases. Moreover, as the cache space increases, the effect will be further improved. As for the inductive task, even the prediction result is not as good as the transductive one, our model can always keep its leading role. 

\begin{table}[]
\renewcommand\arraystretch{1.2}
\centering
\caption{The performance of MOOC dataset with two different memory update periods, (i.e., 24 hours and 1 hour), and different user sets.}
\label{tab:my-table1}
\scalebox{0.75}{%
\begin{tabular}{cc|cccc|cccc}
\hline
\multicolumn{2}{c|}{Dataset} & \multicolumn{4}{c|}{Transductive} & \multicolumn{4}{c}{Inductive} \\ \hline
\multicolumn{2}{c|}{Model} & \multicolumn{2}{c}{TGN-L} & \multicolumn{2}{c|}{ATAGNN} & \multicolumn{2}{c}{TGN-L} & \multicolumn{2}{c}{ATAGNN} \\ \hline
\multicolumn{2}{c|}{Cache Size} & 24 & 1 & 24 & 1 & 24 & 1 & 24 & 1 \\ \hline
\multicolumn{1}{c|}{\multirow{3}{*}{\begin{tabular}[c]{@{}c@{}}Past \\ users\end{tabular}}} & 10 & 47.106 & 46.527 & 49.074 & 48.848 & 50.555 & 49.374 & 52.870 & 53.182 \\
\multicolumn{1}{c|}{} & 15 & 61.944 & 62.097 & 65.190 & 65.860 & 63.673 & 62.290 & 67.598 & 67.025 \\
\multicolumn{1}{c|}{} & 20 & 75.057 & 74.486 & 79.170 & 77.979 & 76.384 & 72.942 & 81.114 & 79.947 \\ \hline
\multicolumn{1}{c|}{\multirow{3}{*}{\begin{tabular}[c]{@{}c@{}}Future\\ users\end{tabular}}} & 10 & 53.111 & 53.020 & 51.620 & 52.234 & 55.311 & 53.983 & 56.280 & 57.807 \\
\multicolumn{1}{c|}{} & 15 & 68.755 & 68.228 & 69.404 & 70.314 & 70.040 & 67.640 & 71.896 & 75.235 \\
\multicolumn{1}{c|}{} & 20 & 81.548 & 80.788 & 83.284 & 84.685 & 80.078 & 77.653 & 84.684 & 87.397 \\ \hline
\end{tabular}%
}
\end{table}

On the other hand, Fig. \ref{fig_re4} displays the comparison between our model and TGN-L concerning about the Wikipedia dataset with various update periods. The performance of both models gradually decreases with the increase of update period, and ATAGNN finally surpasses TGN-L at an update frequency of every hour. In other words, our model relies more on the history messages and an appropriate memory update is more essential and influential for our ATAGNN model when TGN-L also owns a similar prediction power. 

Moreover, we also examine the performance of MOOC dataset with different memory update period (i.e., $T_u = 24,\ 1$). As shown in Table \ref{tab:my-table1}, when the predictions are carried out based on the users that have requested content in the last one hour, ATAGNN can always precede TGN-L. However, unlike the results in Wikipedia, the performance of MOOC with $T_u=24$ is better than that with $T_u = 1$. On the other hand, if we have prior knowledge about the users, we can discover that the results have the same trend as Wikipedia and are much better than the former setting. We also discover that the distribution of users requesting content within two consecutive hours in each hour is more fluctuating in MOOC, which brings more errors when the prediction performance is not as good as in Wikipedia.

\section{Conclusion}\label{sec6}
In this paper, we develop an AoI-based temporal attention graph neural network (ATAGNN) to maximize the precision of users' interest prediction in ICN. By aggregating the history interaction messages with self-attention mechanism, the model is able to generate a vector with rich temporal features. Furthermore, in order to tackle the problem of staleness, the concept of AoI is specifically introduced to exclude stale information for better refining history. The results based on two real-world datasets prove the superiority of the ATAGNN model over other neural network models. Because of its superior performance, a caching strategy totally based on the ATAGNN's prediction results with an appropriate memory update period also wins a great improvement compared to the traditional algorithms or the best baseline based method. 

\begin{appendices}
\section{Proof of the Effectiveness of the Age Filter}\label{APP:a}
{
According to the results from TGN-L and TGN-M \cite{rossi2020temporal} in Table \ref{table_1}, we can observe that the most recent message is more important on the final predictions than averaging all history features. On the one hand, the difference may derive from that averaging smooths historical features too much. On the other hand, the influence of a message may be inversely proportional to its age and some are too old to be positive. The results of Table \ref{tab:my-table} also prove this. 

\begin{theorem}
We further assume that we can mapping the age of information to the effectiveness to the current moment with a function $\mathcal{P}(a_m)$, where $a_m$ is the age of $m$-th history. Due to the self-attention mechanism, all the chosen messages' influence can be denoted as $\sum^{N}_{m=0}\beta_m \mathcal{P}(a_m)$, where $\beta_m$ is the attention coefficient we obtain from Eq. \eqref{eq3.1} and $N$ is the selected number. We believe when the proper threshold $a_M$ satisfy $\mathcal{P}(a_M) \leq \frac{1}{Z} \sum^{N'}_{m'=0}\Delta_{\beta_m'} \mathcal{P}(a_{m'})$, it will filter useful historical messages and improve our final predictions.

\end{theorem}

\begin{proof}
If we want to filter the valuable information with its age adaptively, the total effectiveness of historical information turns into $\sum^{N'}_{m'=0}\beta'_{m'} \mathcal{P}(a_{m'})$. And $N \geq N' = M-1$. If the final results improve, we have the follow condition:
\begin{equation}
\begin{aligned}
\label{oeq}
\sum^{N'}_{m'=0}\beta'_{m'} \mathcal{P}(a_{m'}) - \sum^{N}_{m=0}\beta_m \mathcal{P}(a_m) \geq 0\\
\end{aligned}
\end{equation}
And it can also be considered as:
\begin{equation}
    \begin{aligned}
    \sum^{N'}_{m'=0}(\beta'_{m'} - \beta_{m'} )\mathcal{P}(a_{m'}&) -\sum^{N}_{m=M}\beta_m \mathcal{P}(a_m) \geq 0\\
    \sum^{N'}_{m'=0}(\beta'_{m'}  - \beta_{m'} )\mathcal{P}(a&_{m'}) \geq \sum^{N}_{m=M}\beta_m \mathcal{P}(a_m) 
    \end{aligned}
\end{equation}
where $a_M$ is the threshold that we need. To solve this inequality, we can relax the components at the right side:
\begin{equation}
\label{relax}
\begin{aligned}
\sum^{N'}_{m'=0}(\beta'_{m'} - \beta_{m'}) \mathcal{P}(a_{m'})& \geq (N-M+1)\hat{\beta} \mathcal{P}(a_M)\\
\mathcal{P}(a_M) \leq \frac{1}{Z} &\sum^{N'}_{m'=0}\Delta_{\beta_m'} \mathcal{P}(a_{m'})
\end{aligned}
\end{equation}
where $\hat{\beta}$ is the maximum value in $a_m$ for $m\in (M, M+1 ,\cdots,N)$ and $Z=(N-M+1)\hat{\beta}$, $\Delta_{\beta_m'} = \beta'_{m'} - \beta_{m'}$. Moreover, due to the negative correlation between AoI and its effectiveness, $\mathcal{P}(a_M)$ is also the maximum one. Obviously, if Eq. \eqref{relax} is satisfied, the \eqref{oeq} will be satisfied as well. 
\end{proof}

\begin{remark}
In practice, we find that we can solve this inequality by using a two-layer MLP. We hope the first layer can achieve the component of right side and the second layer can fit the inverse function of the mapping between age and effectiveness. On the other hand, Eq. \eqref{eq5} is adopted to make the backpropagation much easier.
\end{remark}
}
\end{appendices}

\bibliographystyle{IEEEtran}
\bibliography{article}

\begin{thebibliography}{10}
\providecommand{\url}[1]{#1}
\csname url@samestyle\endcsname
\providecommand{\newblock}{\relax}
\providecommand{\bibinfo}[2]{#2}
\providecommand{\BIBentrySTDinterwordspacing}{\spaceskip=0pt\relax}
\providecommand{\BIBentryALTinterwordstretchfactor}{4}
\providecommand{\BIBentryALTinterwordspacing}{\spaceskip=\fontdimen2\font plus
\BIBentryALTinterwordstretchfactor\fontdimen3\font minus
  \fontdimen4\font\relax}
\providecommand{\BIBforeignlanguage}[2]{{%
\expandafter\ifx\csname l@#1\endcsname\relax
\typeout{** WARNING: IEEEtran.bst: No hyphenation pattern has been}%
\typeout{** loaded for the language `#1'. Using the pattern for}%
\typeout{** the default language instead.}%
\else
\language=\csname l@#1\endcsname
\fi
#2}}
\providecommand{\BIBdecl}{\relax}
\BIBdecl

\bibitem{cisco2020cisco}
U.~Cisco, ``Cisco annual internet report (2018--2023) white paper,''
  \emph{Online](accessed March 26, 2021) https://www. cisco.
  com/c/en/us/solutions/collateral/executive-perspectives/annual-internet-report/whitepaper-c11-741490.
  html}, 2020.

\bibitem{li2017cooperative}
Q.~Li, W.~Shi, X.~Ge, and Z.~Niu, ``Cooperative edge caching in
  software-defined hyper-cellular networks,'' \emph{IEEE Journal on Selected
  Areas in Communications}, vol.~35, no.~11, pp. 2596--2605, 2017.

\bibitem{gu2020distributed}
H.~Gu and H.~Wang, ``A distributed caching scheme using non-cooperative game
  for mobile edge networks,'' \emph{IEEE Access}, vol.~8, pp.
  142\,747--142\,757, 2020.

\bibitem{cha2007tube}
M.~Cha, H.~Kwak, P.~Rodriguez, Y.-Y. Ahn, and S.~Moon, ``I tube, you tube,
  everybody tubes: analyzing the world's largest user generated content video
  system,'' in \emph{Proceedings of the 7th ACM SIGCOMM conference on Internet
  measurement}, 2007, pp. 1--14.

\bibitem{liao2020cognitive}
S.~Liao, J.~Wu, J.~Li, A.~K. Bashir, S.~Mumtaz, A.~Jolfaei, and N.~Kvedaraite,
  ``Cognitive popularity based {AI} service sharing for software-defined
  information-centric networks,'' \emph{IEEE Transactions on Network Science
  and Engineering}, vol.~7, no.~4, pp. 2126--2136, 2020.

\bibitem{yang2018content}
P.~Yang, N.~Zhang, S.~Zhang, L.~Yu, J.~Zhang, and X.~Shen, ``Content popularity
  prediction towards location-aware mobile edge caching,'' \emph{IEEE
  Transactions on Multimedia}, vol.~21, no.~4, pp. 915--929, 2018.

\bibitem{paschos2018role}
G.~S. Paschos, G.~Iosifidis, M.~Tao, D.~Towsley, and G.~Caire, ``The role of
  caching in future communication systems and networks,'' \emph{IEEE Journal on
  Selected Areas in Communications}, vol.~36, no.~6, pp. 1111--1125, 2018.

\bibitem{zhang2019sdn}
Z.~Zhang, C.-H. Lung, M.~St-Hilaire, and I.~Lambadaris, ``An sdn-based caching
  decision policy for video caching in information-centric networking,''
  \emph{IEEE Transactions on Multimedia}, vol.~22, no.~4, pp. 1069--1083, 2019.

\bibitem{lederer2014adaptive}
S.~Lederer, C.~Mueller, C.~Timmerer, and H.~Hellwagner, ``Adaptive multimedia
  streaming in information-centric networks,'' \emph{IEEE Network}, vol.~28,
  no.~6, pp. 91--96, 2014.

\bibitem{ahlgren2012survey}
B.~Ahlgren, C.~Dannewitz, C.~Imbrenda, D.~Kutscher, and B.~Ohlman, ``A survey
  of information-centric networking,'' \emph{IEEE Communications Magazine},
  vol.~50, no.~7, pp. 26--36, 2012.

\bibitem{somuyiwa2018reinforcement}
S.~O. Somuyiwa, A.~Gy{\"o}rgy, and D.~G{\"u}nd{\"u}z, ``A
  reinforcement-learning approach to proactive caching in wireless networks,''
  \emph{IEEE Journal on Selected Areas in Communications}, vol.~36, no.~6, pp.
  1331--1344, 2018.

\bibitem{chen2020content}
Q.~Chen, W.~Wang, F.~R. Yu, M.~Tao, and Z.~Zhang, ``Content caching oriented
  popularity prediction: A weighted clustering approach,'' \emph{IEEE
  Transactions on Wireless Communications}, vol.~20, no.~1, pp. 623--636, 2020.

\bibitem{lee2001lrfu}
D.~Lee, J.~Choi, J.-H. Kim, S.~H. Noh, S.~L. Min, Y.~Cho, and C.~S. Kim,
  ``{LRFU}: A spectrum of policies that subsumes the least recently used and
  least frequently used policies,'' \emph{IEEE transactions on Computers},
  vol.~50, no.~12, pp. 1352--1361, 2001.

\bibitem{dan1990approximate}
A.~Dan and D.~Towsley, ``An approximate analysis of the {LRU} and {FIFO} buffer
  replacement schemes,'' in \emph{Proceedings of the 1990 ACM SIGMETRICS
  conference on Measurement and modeling of computer systems}, 1990, pp.
  143--152.

\bibitem{zhou2020graph}
J.~Zhou, G.~Cui, S.~Hu, Z.~Zhang, C.~Yang, Z.~Liu, L.~Wang, C.~Li, and M.~Sun,
  ``Graph neural networks: A review of methods and applications,'' \emph{AI
  Open}, vol.~1, pp. 57--81, 2020.

\bibitem{velivckovic2017graph}
P.~Veli{\v{c}}kovi{\'c}, G.~Cucurull, A.~Casanova, A.~Romero, P.~Lio, and
  Y.~Bengio, ``Graph attention networks,'' in \emph{International Conference on
  Learning Representations}, 2017.

\bibitem{fan2019graph}
W.~Fan, Y.~Ma, Q.~Li, Y.~He, E.~Zhao, J.~Tang, and D.~Yin, ``Graph neural
  networks for social recommendation,'' in \emph{The World Wide Web
  Conference}, 2019, pp. 417--426.

\bibitem{Xu2020Inductive}
D.~Xu, chuanwei ruan, evren korpeoglu, sushant kumar, and kannan achan,
  ``Inductive representation learning on temporal graphs,'' in
  \emph{International Conference on Learning Representations}, 2020.

\bibitem{sun2017update}
Y.~Sun, E.~Uysal-Biyikoglu, R.~D. Yates, C.~E. Koksal, and N.~B. Shroff,
  ``Update or wait: How to keep your data fresh,'' \emph{IEEE Transactions on
  Information Theory}, vol.~63, no.~11, pp. 7492--7508, 2017.

\bibitem{ming2014age}
Z.~Ming, M.~Xu, and D.~Wang, ``Age-based cooperative caching in
  information-centric networking,'' in \emph{2014 23rd International Conference
  on Computer Communication and Networks (ICCCN)}.\hskip 1em plus 0.5em minus
  0.4em\relax IEEE, 2014, pp. 1--8.

\bibitem{wu2014scaling}
Y.~Wu, C.~Wu, B.~Li, L.~Zhang, Z.~Li, and F.~C. Lau, ``Scaling social media
  applications into geo-distributed clouds,'' \emph{IEEE/ACM Transactions On
  Networking}, vol.~23, no.~3, pp. 689--702, 2014.

\bibitem{zhang2017ppc}
Y.~Zhang, X.~Tan, and W.~Li, ``{PPC}: Popularity prediction caching in icn,''
  \emph{IEEE Communications Letters}, vol.~22, no.~1, pp. 5--8, 2017.

\bibitem{mehrizi2019feature}
S.~Mehrizi, A.~Tsakmalis, S.~Chatzinotas, and B.~Ottersten, ``A feature-based
  bayesian method for content popularity prediction in edge-caching networks,''
  in \emph{2019 IEEE Wireless Communications and Networking Conference
  (WCNC)}.\hskip 1em plus 0.5em minus 0.4em\relax IEEE, 2019, pp. 1--6.

\bibitem{lee2020t}
S.~Lee, I.~Yeom, and D.~Kim, ``T-caching: enhancing feasibility of in-network
  caching in icn,'' \emph{IEEE Transactions on Parallel and Distributed
  Systems}, vol.~31, no.~7, pp. 1486--1498, 2020.

\bibitem{DBLP:journals/corr/HidasiKBT15}
B.~Hidasi, A.~Karatzoglou, L.~Baltrunas, and D.~Tikk, ``Session-based
  recommendations with recurrent neural networks,'' in \emph{International
  Conference on Learning Representations}, 2016.

\bibitem{wu2019session}
S.~Wu, Y.~Tang, Y.~Zhu, L.~Wang, X.~Xie, and T.~Tan, ``Session-based
  recommendation with graph neural networks,'' in \emph{Proceedings of the AAAI
  Conference on Artificial Intelligence}, vol.~33, no.~01, 2019, pp. 346--353.

\bibitem{wang2020multi}
X.~Wang, R.~Wang, C.~Shi, G.~Song, and Q.~Li, ``Multi-component graph
  convolutional collaborative filtering,'' in \emph{Proceedings of the AAAI
  Conference on Artificial Intelligence}, vol.~34, no.~04, 2020, pp.
  6267--6274.

\bibitem{sankar2019dynamic}
A.~Sankar, Y.~Wu, L.~Gou, W.~Zhang, and H.~Yang, ``Dynamic graph representation
  learning via self-attention networks,'' \emph{arXiv preprint
  arXiv:1812.09430}, 2018.

\bibitem{skardinga2021foundations}
J.~Skardinga, B.~Gabrys, and K.~Musial, ``Foundations and modelling of dynamic
  networks using dynamic graph neural networks: A survey,'' \emph{IEEE Access},
  2021.

\bibitem{trivedi2019dyrep}
R.~Trivedi, M.~Farajtabar, P.~Biswal, and H.~Zha, ``Dyrep: Learning
  representations over dynamic graphs,'' in \emph{International Conference on
  Learning Representations}, 2019.

\bibitem{vaswani2017attention}
A.~Vaswani, N.~Shazeer, N.~Parmar, J.~Uszkoreit, L.~Jones, A.~N. Gomez,
  {\L}.~Kaiser, and I.~Polosukhin, ``Attention is all you need,'' in
  \emph{Advances in neural information processing systems}, 2017, pp.
  5998--6008.

\bibitem{rossi2020temporal}
E.~Rossi, B.~Chamberlain, F.~Frasca, D.~Eynard, F.~Monti, and M.~Bronstein,
  ``Temporal graph networks for deep learning on dynamic graphs,'' \emph{arXiv
  preprint arXiv:2006.10637}, 2020.

\bibitem{hamilton2017inductive}
W.~L. Hamilton, R.~Ying, and J.~Leskovec, ``Inductive representation learning
  on large graphs,'' in \emph{Proceedings of the 31st International Conference
  on Neural Information Processing Systems}, 2017, pp. 1025--1035.

\bibitem{he2016deep}
K.~He, X.~Zhang, S.~Ren, and J.~Sun, ``Deep residual learning for image
  recognition,'' in \emph{Proceedings of the IEEE conference on computer vision
  and pattern recognition}, 2016, pp. 770--778.

\bibitem{cho2014learning}
K.~Cho, B.~Van~Merri{\"e}nboer, C.~Gulcehre, D.~Bahdanau, F.~Bougares,
  H.~Schwenk, and Y.~Bengio, ``Learning phrase representations using rnn
  encoder-decoder for statistical machine translation,'' \emph{arXiv preprint
  arXiv:1406.1078}, 2014.

\bibitem{kosta2017age}
A.~Kosta, N.~Pappas, and V.~Angelakis, ``Age of information: A new concept,
  metric, and tool,'' \emph{Foundations and Trends in Networking}, vol.~12,
  no.~3, pp. 162--259, 2017.

\end{thebibliography}
\end{document}